\documentclass[10pt,journal,compsoc]{IEEEtran}
\ifCLASSOPTIONcompsoc
\usepackage[nocompress]{cite}
\else
\usepackage{cite}
\fi
\ifCLASSINFOpdf
\else
\fi
\usepackage{bm}
\usepackage{subcaption}
\usepackage{amsfonts}
\usepackage{multirow}
\usepackage{lineno}
\usepackage{xcolor}
\usepackage{comment}
\usepackage{amsmath} 
\usepackage{amsthm}
\newtheorem{theorem}{Theorem}

\usepackage{enumitem}
\usepackage{adjustbox}  
\usepackage[linesnumbered,ruled,boxed]{algorithm2e}
\SetKwInput{KwInput}{Input}
\SetNlSty{}{}{}
\usepackage{setspace}
\SetAlgorithmName{Algorithm}{algorithm}{List of Algorithms}

\usepackage{bm}
\DeclareCaptionLabelFormat{nonbold-parentheses}{\normalfont(#2)}
\begin{document}

\title{FairDgcl: Fairness-aware Recommendation with Dynamic Graph Contrastive Learning}

\author{Wei~Chen,
	Meng~Yuan,
	Zhao~Zhang,
	Ruobing~Xie,
	Fuzhen~Zhuang$^*$,
	Deqing~Wang,
	Rui~Liu
	\IEEEcompsocitemizethanks{\IEEEcompsocthanksitem Wei Chen and Meng Yuan are with Institute of Artificial Intelligence, Beihang University, Beijing 100191, China.
		E-mail: \{chenwei23,yuanmeng97\}$@$buaa.edu.com
				\IEEEcompsocthanksitem  Zhao Zhang is with Institute of Computing Technology, Chinese Academy of Sciences, Beijing, China.
		E-mail:  \{zhangzhao2021\}@ict.ac.cn
						\IEEEcompsocthanksitem  Ruobing Xie is with WeChat, Tencent, Beijing, China.
		E-mail:  \{xrbsnowing\}@163.com
		\IEEEcompsocthanksitem Fuzhen Zhuang is with Institute of Artificial Intelligence, Beihang University, Beijing 100191, China, and Zhongguancun Laboratory, Beijing, China
		E-mail: zhuangfuzhen@buaa.edu.cn
			\IEEEcompsocthanksitem Deqing Wang and Rui Liu are with School of Computer, Beihang University,  100191, Beijing, China
	E-mail: \{dqwang,lr\}$@$buaa.edu.cn}
	
	\thanks{ ${*}$ indicates corresponding author.}}


\markboth{XXXX}%
{Shell \MakeLowercase{\textit{et al.}}: Bare Demo of IEEEtran.cls for Computer Society Journals}
\IEEEtitleabstractindextext{%
\begin{abstract}
		As trustworthy AI continues to advance, the fairness issue in recommendations has received increasing attention. A recommender system is considered unfair when it produces unequal outcomes for different user groups based on user-sensitive attributes (e.g., age, gender). Some researchers have proposed data augmentation-based methods aiming at alleviating user-level unfairness by altering the skewed distribution of training data among various user groups.
Despite yielding promising results, they often rely on fairness-related assumptions that may not align with reality, potentially reducing the data quality and negatively affecting model effectiveness.
To tackle this issue, in this paper, we study how to implement high-quality data augmentation to improve recommendation fairness.
Specifically, we propose \textbf{FairDgcl}, a dynamic graph adversarial contrastive learning framework  aiming at improving fairness in recommender system. First,
FairDgcl develops an adversarial contrastive network with a view generator and a view discriminator to learn generating fair augmentation strategies in an adversarial style. Then, we propose two dynamic, learnable models to generate contrastive views within contrastive learning framework, which automatically fine-tune the augmentation strategies. Meanwhile, we theoretically show that FairDgcl can simultaneously generate enhanced representations that possess both fairness and accuracy.
Lastly, comprehensive experiments conducted on four real-world datasets demonstrate the effectiveness of the proposed FairDgcl. 
The code can be found at https://github.com/cwei01/FairDgcl.
\end{abstract}

\begin{IEEEkeywords}
	fairness, recommender system, graph contrastive learning, adversarial training
\end{IEEEkeywords}}

\maketitle
\IEEEdisplaynontitleabstractindextext

\IEEEpeerreviewmaketitle

\IEEEraisesectionheading{\section{Introduction}
	\label{sec.introduction}}
\IEEEPARstart{R}{ecommender}  systems play a vital role in numerous online applications such as e-commerce platforms and entertainment apps, significantly enriching the user experience ~\cite{aljukhadar2012using,qian2013personalized,ko2022survey,wu2023personalized}.
Recently, Graph Neural Networks (GNNs)~\cite{he2020lightgcn,chen2020revisiting,gao2023survey} have significantly boosted recommender systems performance by efficiently leveraging user-item interactions. 
However, since recommender systems are applied to human-centered applications, excessive pursuit of recommendation accuracy may harm the fairness experience of some users~\cite{li2021user,	ling2023learning}.
A recommender system is unfair when it yields unequal outcomes for different user groups based on user-sensitive attributes. (e.g., age, gender).
For instance, equally skilled men are favored over women in job recommender~\cite{lambrecht2019algorithmic}, and active users tend to receive better recommendations than inactive ones~\cite{melchiorre2021investigating}.
\begin{figure}[t]
	\centering
	\setlength{\fboxrule}{0.pt}
	\setlength{\fboxsep}{0.pt}
	\fbox{
		\includegraphics[width=0.95\linewidth]{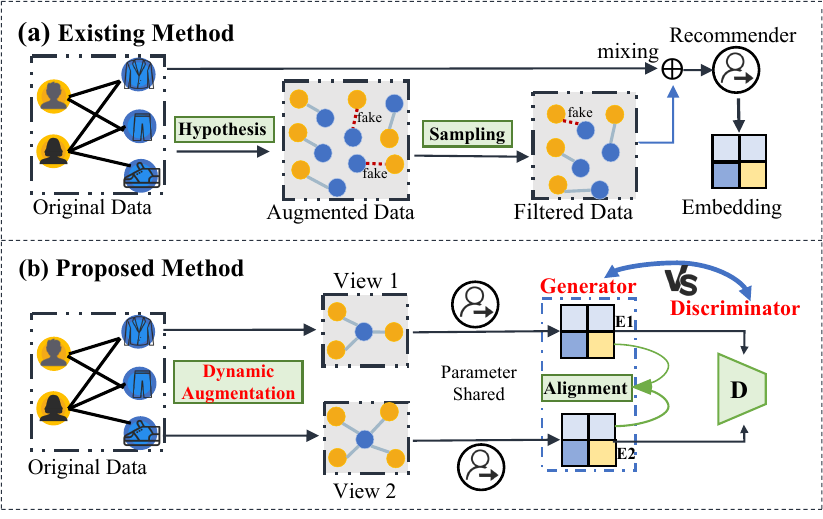}
	}
	\caption{Comparing data augmentation paradigms
		between existing and proposed fairness-aware recommenders.
	}
	\label{intro}
\end{figure}

A recent study~\cite{chen2023improving} uses GNN models to measure the distribution differences between the two user groups in both training data and recommendation outcomes. The results reveal substantial disparities in recommendation outcomes among user groups, often surpassing those observed in the training data. 
This implies that the GNN models might not only inherit unfairness but also amplify it from the data. Consequently, the main concern with the unfairness issue stems from the biases inherent in the training data. Based on this premise, 
research efforts~\cite{rastegarpanah2019fighting,iosifidis2019fae,chen2023improving,ying2023camus} have been devoted to improving recommendation fairness from the perspective of data augmentation. 
Usually, they first generate some augmented interaction data based on some hypothetical labels\footnote{For example, each user in one group has a similar item preference as the item preference of any user in the remaining group~\cite{chen2023improving}.} to balance the uneven interactions between different user groups, and then selectively mix them into the original data for joint training (see Figure~\ref{intro}). Despite achieving some promising results relying on these synthetic data, their applicability may be limited due to users' different personal preferences in real-world recommendation scenarios, potentially reducing data quality and negatively affecting model effectiveness.

Instead of relying on additional assumed labels for generating augmented data, contrastive learning~\cite{khosla2020supervised} (CL) employs the input data itself as the supervision signal. It constructs an augmented data pair to teach the model to compare their similarity, and has shown competitive performance in various machine learning tasks~\cite{jaiswal2020survey,kim2021self,schiappa2023self,huang2023adversarial}.
Impressed by the exceptional performance of CL paradigm, recent studies~\cite{wu2021self,yu2023xsimgcl,xia2023automated} have introduced graph contrastive learning (GCL) into recommender systems. Its key is to perform data augmentation to generate different views and contrastive learning tasks.
While GCL can indeed solve the problem of low-quality augmented data to some extent, it still faces two main challenges when tackling fairness issues in recommender systems:

(1) \textbf{Fairness-oriented Data Augmentation.}
Most current GCL methods heavily depend on unsupervised data augmentation techniques, like randomly removing nodes or edges~\cite{jaiswal2020survey,yu2023xsimgcl}. However, relying solely on these methods may not be sufficient to guarantee fair recommendation results. An intuitive way to promote fairness in this process is by setting fairness constraints~\cite{spinelli2021fairdrop,chen2023fairgap}, such as removing interactions that contain sensitive information. 
Nonetheless, when dealing with graph data, i) the relationship between nodes and edges can be highly complex, with interactions among them being multi-dimensional and non-linear. Simply removing or altering specific interactions may not provide a comprehensive solution to the unfairness issue. ii) More importantly, discrimination may not manifest overtly. It could be concealed within the inherent feature of nodes or arise as a consequence of how nodes interact with one another. 
This inherent ambiguity significantly complicates the identification and quantification of these biased graph elements.
Therefore, it is highly desirable to explore efficient fairness-oriented data-augmentation methods in GCL.

(2) \textbf{Dynamic Data Augmentation.}
In fairness-oriented recommendation, an important problem lies in striking a balance between recommendation accuracy and fairness~\cite{wu2022selective,chen2023fairgap}. 
This balance is delicate, as efforts to enhance one aspect may adversely impact the other. 
As previously noted, within the same model, post-training recommendation unfairness often surpasses pre-training data's inherent unfairness. This suggests that while aiming for greater accuracy, the model may inadvertently introduce more unfairness, leading to a dynamic balance between accuracy and fairness during the process. Although previous methods have shown promise with single-round data augmentation\footnote{It refers to a one-time modification of data to address recommendation fairness and accuracy simultaneously.}, maintaining this balance consistently is still difficult.
Thus, it is essential to implement a dynamic, multi-round data augmentation approach,
which can offer a more nuanced way of continuously recalibrating the balance between accuracy and fairness. However, the discrete nature of graph data, such as binary adjacency matrix values, complicates maintaining balance during dynamic augmentation.

In light of these challenges, we propose a novel \underline{\textbf{D}}ynamic  \underline{\textbf{g}}raph  \underline{\textbf{c}}ontrastive \underline{\textbf{l}}earning framework to improve recommendation \underline{\textbf{Fair}}-ness (\textbf{FairDgcl}), as shown in Figure~\ref{intro}. Specifically, we develop a graph adversarial contrastive approach to address the first challenge. It consists of two key components: a view generator is responsible for learning fair augmentation strategies and generating fair representation, and a view discriminator evaluates the fairness of these generated views.
They work together adversarially to implicitly optimize model and mitigate unfairness in the input graph.
Then, we introduce two dynamically learned models as the view generator to establish distinct
contrastive views that address the issue of dynamic augmentation.
This approach allows the proposed view generator to provide a more adaptive method for continuously recalibrating the balance between accuracy and fairness. In summary,
the main contributions of this paper are outlined as follows:

\begin{itemize}[leftmargin=*,labelindent=0pt,topsep=2pt,itemsep=1pt]
	\item We investigate the drawbacks of existing data augmented approaches for improving recommendation fairness. To our knowledge, this is the first attempt to employ adversarial contrastive learning  to address the fairness issue in recommender systems.
	
	\item  We present FairDgcl, a novel data-augmented framework that develops a view generator and a view discriminator to generate fair node representation via a minimax adversarial game.  We theoretically prove that FairDgcl can generate improved
	representations with fairness and informativeness.
	\item  Our experimental results demonstrate that our FairDgcl outperforms most baseline models on four datasets in terms of both fairness and accuracy, highlighting its effectiveness.
\end{itemize}

\section{Related Work}
In this section, we will briefly review the fairness-aware recommendation and data augmentation, and discuss the relationship between this study and previous works.
\subsection{Fairness-aware Recommendation}
With the development of trustworthy AI~\cite{thiebes2021trustworthy,wang2022trustworthy,li2023trustworthy}, ensuring algorithmic fairness is crucial, particularly in human-centric recommender systems. 
In the field of recommendation, fairness demands can be categorized into user-oriented~\cite{boratto2023counterfactual,chen2023fairgap,zhu2023path} and item-oriented~\cite{wu2022joint,chen2023fairly,jiang2024item}, depending on the stakeholder being considered. User-side fairness focuses on ensuring that all users get equitable chances at recommendations, regardless of their sensitive attributes. Item-side fairness is concerned with ensuring that all items, regardless of their popularity, are fairly recommended to users.
Furthermore, user-oriented fairness can be categorized into two categories: i) individual fairness, which advocates for equal treatment of every user, ii) group fairness, which aims for equal recommendation opportunities across user groups defined on sensitive attributes. In this paper, we explore fairness-aware recommendations based on the concept of group fairness~\cite{dwork2018group,awasthi2020beyond}. Many fair training methods have been developed to achieve group fairness in this domain. For example, Yao and Huang~\cite{yao2017beyond} propose four metrics to optimize the fairness of recommender systems, achieved by integrating fairness-oriented constraints into the learning objective. Wu et al.~\cite{wu2021learning} leverage adversarial learning to efficiently filter sensitive attributes within user data, while preserving essential information, thereby achieving fair representation learning. Li et al.~\cite{li2021towards} develop a notion of counterfactual fairness rooted in causal theory, and construct a framework for a mixed-structure adversarial model to actualize it. 
Li et al.~\cite{li2021user} introduce a reranking algorithm that ensures the fairness of recommendation systems by effectively solving a 0-1 programming problem. Yang et al.~\cite{yang2023towards} address fairness issues arising from distributional discrepancies in recommender systems using distributionally robust optimization techniques. Although these models have achieved numerous successes, they primarily focus on the algorithm itself. Machine learning problems typically involve two core components: data and algorithm. In this paper, our aim is to enhance recommendation fairness from a data-centric perspective, which is significantly different from prevalent model-based research. 
\subsection{Data Augmentation}
With the success of data augmentation techniques in natural language processing (NLP)~\cite{feng2021survey,chen2023empirical} and computer vision (CV)~\cite{shorten2019survey,xu2023comprehensive} tasks, 
data augmentation techniques applied in recommender system has also received extensive attention. 
In recommendation scenarios,  GCL is an effective data augmentation strategy used to alleviate the data sparsity problem. The key idea of this approach is to create diverse enhanced views and subsequently strive to optimize the similarity of representations across these views.
Furthermore, some studies~\cite{huang2023adversarial,jiang2023adaptive} design learnable data augmentation methods that automatically improve node representation in bipartite graphs within the GCL framework.
However, it should be noted that the methods previously mentioned do not address fairness concerns, and the exploration of fairness-aware data augmentation has been limited in only a few studies. For example, 
Spinelli et al.~\cite{spinelli2021fairdrop} argue that nodes with similar sensitive attributes tend to form connections, leading to biased predictions. Thus, they propose an algorithm to selectively remove edges in graphs, reducing biases and improving fairness in predictive tasks. Chen et al.~\cite{chen2023fairgap} introduce a graph editing methodology inspired by counterfactual thinking~\cite{roese1994functional,wang2021counterfactual}, aimed at generating a fair graph and learning fair node representations. More recently, 
Ying et al.~\cite{ying2023camus} provide an attribute-aware counterfactual data augmentation strategy tailored for minority users, designed to reduce disparities across different user groups. 
Additionally, Chen et al.~\cite{chen2023improving} propose a fairness-aware augmented framework, based on the assumption that users in one group share similar item preferences with users in the other group, aiming to generate synthetic interaction data to balance group interactions for different user groups.  
While achieving improved results relying on these augmented data,
their reliability in real scenarios is questionable due to varying user preferences, potentially degrading data quality and impairing model effectiveness. Hence, we will explore high-quality data augmentation strategies to improve fairness in recommender systems. In this work, we consider contrastive learning as a data augmentation technique, utilizing the input data itself as the supervision signal and constructing an augmented data to guide model learning.

\begin{figure*}[t]
	\vspace{-0pt}
	\centering
	\setlength{\fboxrule}{0.pt}
	\setlength{\fboxsep}{0.pt}
	\fbox{
		\includegraphics[width=1\linewidth]{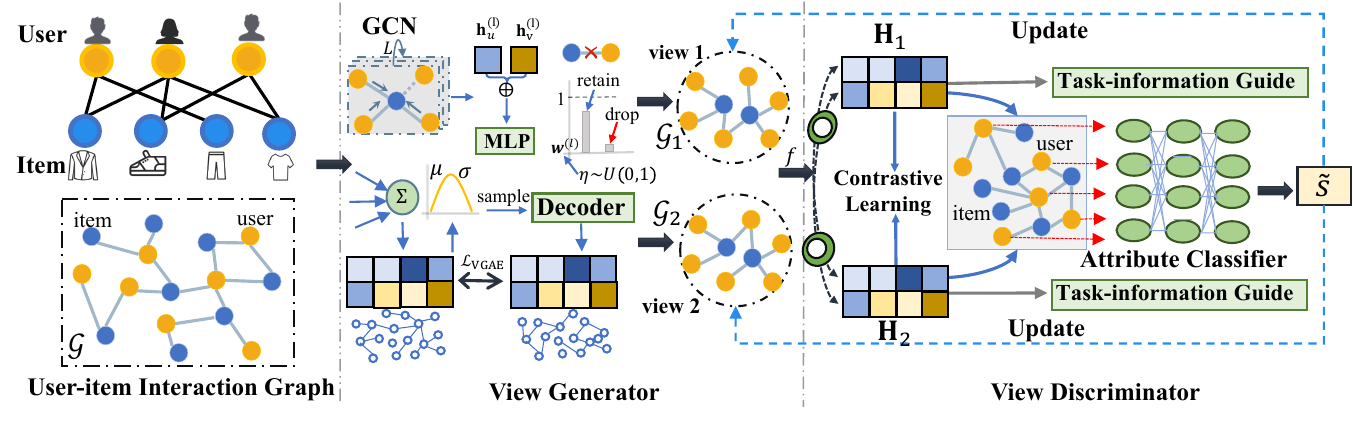}
	}
	\caption{A depiction of proposed FairDgcl framework. It is specifically designed to enhance fairness in input bipartite graphs automatically through two modules: the view generator and the view discriminator. The former learns fair augmentation strategies to generate augmented views, while the latter evaluates whether these generated views achieve sufficient fairness.
	}
	\label{model}
\end{figure*}

\section{Preliminaries}
In this section, we present preliminary knowledge. We first introduce typical neural graph collaborative filtering methods and then describe the notation of user-oriented fairness.
\subsection{Neural Graph Collaborative Filtering}
Collaborative filtering (CF) is one of the core components of recommendation. Its purpose is to model users' potential preferences through observed user-item interactions. 
A typical recommender system comprises two sets of entities: a user set $\mathcal{U}$ and an item set $\mathcal{V}$. 
In the context of graph collaborative filtering, the interactive matrix is transformed into a bipartite graph denoted as $\mathcal{G} = <\mathcal{U} \cup \mathcal{V}, \mathcal{E}>$, where $\mathcal{E}$ represents the corresponding edges.

	
	In line with most prior studies~\cite{chen2023fairgap,huang2023adversarial,jiang2023adaptive}, we adopt LightGCN~\cite{he2020lightgcn} as the backbone GNN, leveraging its message-passing mechanism to aggregate and summarize node embeddings for both the user $u$ and item $v$. Subsequently, we calculate the predicted user $u$'s preference for item $v$ as the inner product between the respective embeddings, defined as $y\left(u, v\right)=\mathbf{h}_u^T \mathbf{h}_v$. During the training process, we adopt the Bayesian Personalization Ranking (BPR) loss to learn
	the embeddings for each user and item:
	\begin{equation}
		\mathcal{L}_{\mathrm{BPR}} (u, v^{+}, v^{-}) = -\log \sigma\left(y\left(u, v^{+}\right)-y\left(u, v^{-}\right)\right),
		\label{bprloss}
	\end{equation}
	where $\left(u, v^{+}\right)$ represents the positive interactions and $\left(u, v^{-}\right)$ represent random negative interactions,
	$\sigma$ denotes the sigmoid function. The $\mathrm{BPR}$ loss aims to boost the prediction scores of observed user-item interactions above those of unobserved interactions.

	\subsection{User-oriented Fairness}
	
	In recommendation, 
	user-oriented fairness aims to provide a balanced and personalized recommendation experience by minimizing disparities among different users. More specifically, user fairness can be categorized into two main aspects based on target audiences: individual fairness and group fairness. Individual fairness treats similar users equally, while group fairness aims to reduce service disparities among user groups for satisfactory recommendations.
	This work primary emphasis lies in addressing user-level group fairness while considering two distinct user groups. Initially, we divide users $\mathcal{U}$ into two groups based on their sensitive attributes, namely $\mathcal{U}_{\bm{s}=0}$ and $\mathcal{U}_{\bm{s}=1}$. 
	Subsequently, we utilize the group fairness metric to improve fairness~\cite{chen2023improving,chen2023fairgap,ling2023learning}, a widely accepted measure designed to minimize performance disparities between different user groups. Formally, the group fairness is expressed as follows: 
	\begin{equation}
		\Phi =\left|\frac{1}{\left|\mathcal{U}_{\bm{s}=0}\right|} \sum_{u \in \mathcal{U}_{\bm{s}=0}}\!\!\! \mathcal{F}_{u}-\frac{1}{\left|\mathcal{U}_{\bm{s}=1}\right|} \sum_{u \in \mathcal{U}_{\bm{s}=1}} \!\!\!\mathcal{F}_{u}\right|,
		\label{phi}
	\end{equation}
	where $\mathcal{F}_{u}$ is the recommendation performance for user $u$ (e.g., recall, ndcg).
	In practice, optimizing Eq.~(\ref{phi}) in recommendation models is challenging due to the non-differentiable nature of conventional ranking-based metrics such as recall and ndcg.
	To tackle this problem, prior research has proposed two main approaches. The first approach focuses on creating surrogate loss functions to make the fairness metric differentiable, enabling direct optimization~\cite{yang2023towards}. This method aims to attain fairness while preserving model performance.
	The second approach adopts a counterfactual fairness perspective~\cite{li2021towards} and suggests that group fairness can be achieved by removing users' sensitive attributes. This means that during the recommendation process, the model either disregards or reduces its reliance on these attributes to promote fairness. 
	In this paper, we align with this second solution, aiming to achieve group fairness by filtering users' sensitive attributes.

	\section{METHODOLOGY}
	In this section, we first present the overview of the proposed FairDgcl model (in Figure~\ref{model}), and then bring forward the details of its major modules. Finally, we introduce the optimization processes.

	\subsection{Model Overview}
	As shown in Figure~\ref{model}, FairDgcl consists of two main modules, namely view generator and view discriminator.
	Specifically, the view generator learns fair augmentation strategies and generates fair representations.
	Then, the view discriminator is designed to evaluate whether the augmented views are fair enough. The view generator and the view discriminator are trained in an adversarial style to generate high-quality views. These augmented views are used to train fair and effective user and item representations.
	Table~\ref{notation} outlines the primary symbols and their meanings in this paper.
	
	\begin{table}[t]
		\caption{Notations and definitions.}
		\centering
		\resizebox{0.75\linewidth}{!}{
		\begin{tabular}{c|c}
			\hline 
			\textbf{Notations} & 	\textbf{Definitions} \\ \hline \hline
			$\mathcal{G}$    & The original graph     \\
			$\mathcal{U}$    & The set of all users     \\
			$\mathcal{V}$  & The set of all items     \\
			$\mathcal{E}$    & The set of all edges    \\ 
						$\mathcal{X}$    & The node feature vector   \\ 
									$\mathcal{A}$    & The adjacency matrix    \\ 
			$\mathcal{G}_{1}$    & The first generative graph  \\ 
			$\mathcal{G}_{2}$   & The second generative graph   \\ 
			$\mathbf{H}_{1}$    & The embedding of $\mathcal{G}_{1}$    \\ 
			$\mathbf{H}_{2}$   & The embedding of $\mathcal{G}_{2}$     \\ 
			$\bm{s}$    & The real sensitive attribute   \\ 
			$\tilde{\bm{s}}$    & The predicted sensitive attribute   \\ 
			$f$   & The graph encoder   \\ 
			$\theta_{g}$   & The parameter set of view generator   \\ 
			$\theta_{f}$   &  The parameter set of graph encoder $f$ \\ 
			$\theta_{d}$   & The parameter set of view discriminator   \\ 
			
			\hline 
		\end{tabular}}
		\label{notation}
	\end{table}

	\subsection{View Generator}
	Given a graph $\mathcal{G} = <\mathcal{U} \cup \mathcal{V}, \mathcal{E}>$, where $\mathcal{U}$ ($\mathcal{V}$) represents user (item) nodes, and $\mathcal{E}$ denotes
	corresponding edges. The view generator is designed to generate two augmented views. 
	As previously discussed, fairness-aware data augmentation requires considering the dynamic balance between accuracy and fairness. Accordingly,
	we propose to use two learnable models as view generators to generate adaptive views for GCL.
	It should be noted that it is crucial to generate diverse and discriminative views with GCL. If two views produced by the same model have identical distributions, there's a risk of model collapse~\cite{jiang2023adaptive}, potentially leading to inaccurate comparison optimization.
	To address this issue, we propose employing two separate models to augment the user-item graph from distinct perspectives.
	Precisely, we leverage two learned models in FairDgcl: a recognition model and a generative model. 
	The recognition model uses the graph's topology to remove unfair interactions in the user-item graph and create a less biased view. 
	Meanwhile, the generative model focuses on reconstructing fair views by considering graph distributions.
	In the following sections, we will provide detailed descriptions of both methods.
	\subsubsection{\textbf{Recognition Model View Generator.}} 
	GNNs employ message passing to capture node representations. However, in real-world situations, these interactions can introduce unfairness issues due to potential biases. For example, if user-item interactions are influenced by sensitive attributes, the model may wrongly interpret them as inherent preferences, leading to biased outputs.
	Hence, we propose a parameterized recognition model~\cite{ye2021sparse,wang2023knowledge} designed to reduce unfairness by directly removing interactions associated with sensitive attributes. This can reduce the model's reliance on sensitive attributes and better reflect users' real interests.
	
	Technically,  the recognition model computes a real-valued fairness weight $\bm{w}_{e}^{l}$ and determines a binary sampling probability $\bm{p}_{e}^{l} \in \{0,1\}$ for every edge $e \in \mathcal{E}$ in $l$-th layer. It's worth mentioning that the edge  $e$ will be retained if $\bm{p}_{e}^{l}=1$ and discarded otherwise. Formally, the weight $\bm{w}_{e}^{l}$ is computed as follows:
	\begin{equation}
		\bm{w}_{e}^l=\text{MLP}\left(\left[\mathbf{h}_u^{(l)} \parallel  \mathbf{h}_v^{(l)}\right]\right),
	\end{equation}
	where $e = (u, v)$ represents an edge, $\bm{w}_e^l$ signifies edge fairness importance, MLP stands for multi-layer perception, and $\parallel $ represents concatenation. $\mathbf{h}_u^{(l)}$ and $\mathbf{h}_v^{(l)}$ are user and item representations in the $l$-th layer.
	Note that a higher $\bm{w}_e^l$ suggests that the edge $e$ is more likely to be critical and should be preserved.
	To enable end-to-end optimization and ensure differentiability in the edge dropping procedure, we transform the discrete variable $\bm{p}_{e}^{l}$ into a continuous variable in $(0, 1)$. It can be achieved using the Gumbel-Max reparameterization trick~\cite{luo2020parameterized}, where we define it as:
	\begin{equation}
		\bm{p}_e^l=\text{sigmoid}\left(\left(\log (\eta)-\log (1-\eta)+\bm{w}_e^l\right) / \tau_r\right),
	\end{equation}
	where $\eta$ is a priori constant offset, and temperature hyper-parameter $\tau_r$ is used to control the approximation. As $\tau_r$ approaches 0, $\bm{p}_{e}^{l}$ converges to a binary value. Finally, 
	we denote the augmented interactive graph as $\mathcal{G}_{1}$.

		\subsubsection{\textbf{Generative model as View Generator. }}
		In contrast to the recognition approach in the previous section, this part focuses on improving node representation by utilizing Variational Graph Auto-Encoder (VGAE)~\cite{kipf2016variational} from the reconstruction respective.
		VGAE seamlessly combines variational auto-encoder principles with graph generation techniques, offering an efficient and flexible framework to enhance node representations.
		VGAE consists of two main components: the encoder and the decoder. Encoders typically utilize GCNs to extract node feature representations and project them into a latent space. Specifically, it transforms the node feature vector $\mathcal{X}$ and adjacency matrix $\mathcal{A}$ into two vectors: a mean vector $\mu$ and a standard deviation vector $\sigma$, denoted as $\mu, \log (\sigma^2) = Encoder(\mathcal{X}, \mathcal{A})$.
		The decoder's role is to rebuild the graph based on features, often by predicting edge connections between nodes using the inner product of their latent representations, denoted as $p(\mathcal{X}\!\!\mid\!\! \mathcal{Z})=\text{sigmoid} \left(\mathcal{X} \mathcal{X}^T\right)$.
		Here, $\mathcal{Z}$ represents the node representation drawn from the latent space,
		and sigmoid function used to transform the inner product into a probability. 
		This process ensures a smooth and concise logic flow.
		
		In Figure~\ref{model}, we employ a multi-layer GCN as the encoder to capture graph embeddings and calculate both the mean and standard deviation of these embeddings. The decoder, implemented as an MLP, takes these mean and standard deviation values, along with Gaussian noise, to generate a new graph. Lastly, VGAE is end-to-end trained by maximizing the evidence lower bound (ELBO):
		\begin{equation}
			\mathcal{L}_\mathrm{VGAE}=\mathbb{E}_{q(\mathcal{Z} \mid \mathcal{X}, \mathcal{A})}[\log p(\mathcal{A} \!\!\mid\!\! \mathcal{Z})]-K L[q(\mathcal{Z}\!\! \mid \!\!\mathcal{X}, \mathcal{A}) \| p(\mathcal{Z})],
		\end{equation}
		where the first term represents the reconstruction probability, indicating how likely the decoder can reconstruct the adjacency matrix $\mathcal{A}$ from the latent representation. The second term is the Kullback-Leibler divergence between the variational distribution $q(\mathcal{Z}\!\!\mid\!\! \mathcal{X}, \mathcal{A})$ of the latent representation $\mathcal{Z}$ and the prior distribution $p(\mathcal{Z})$. This divergence serves as a regularization factor, encouraging the latent space to closely match the prior distribution.
		Similarly, we denote the second augmented interactive graph as $\mathcal{G}_{2}$.

		\subsection{View Discriminator} 
		The view discriminator is a user-level attribute classifier for recognizing the generated views. Specifically, the discriminator takes an augmentation view embedding as input and judges whether the view still contains sensitive attributes.
		Considering the presence of two different generative models, ideally, we should train two discriminators independently. However, this approach weakens the connections between the two generated views and also increases the model's parameters\footnote{The complexity is $\mathcal{O}(2 \times|\mathcal{U}| \times d)$), where $|\mathcal{U}| $ denotes the number of users, and $d$ indicates the embedding dimension.}.
		Thus, we merge the two augmented views into one augmented view and use a unified discriminator to simultaneously optimize both augmented representations. 
		More precisely, we employ a parameter-sharing graph encoder $f$ to transform each of the augmented views into a unified embedding space, denoted as $\mathbf{H}_{1} = f(\mathcal{G}_{1})$ and  $\mathbf{H}_{2} = f(\mathcal{G}_{2})$.
		Subsequently, an MLP model serves as the discriminative network to predict the sensitive information based on the fused user representation, that is:
		\begin{equation}
			\tilde{\bm{s}} =\text{MLP}\left(\text{mean} \left(\mathbf{H}_{1}^{(u)}, \mathbf{H}_{2}^{(u)}\right)\right),
		\end{equation}
		where $\text{mean}$ denotes mean pooling, MLP is short for multi-layer perception. $\mathbf{H}_{1}^{(u)}$ and $\mathbf{H}_{2}^{(u)}$ represent the user embeddings in augmented view $\mathcal{G}_{1}$ and $\mathcal{G}_{2}$, respectively.  $\tilde{\bm{s}} \in[0,1] $ is the predicted score of sensitive attribute.
		
		To train the discriminator, we define the attribute label of user $i$ as $\bm{s}_{i}$, and the classification loss is defined as follows:
		\begin{equation}
			\mathcal{L}_{\mathrm{VD}}= -\frac{1}{|\mathcal{U}|} \sum_{i \in \mathcal{U} }\left[\bm{s}_i \log \tilde{\bm{s}}_i+\left(1-\bm{s}_i\right) \log \left(1-\tilde{\bm{s}}_i\right)\right].
			\label{dis}
		\end{equation}

		\noindent
		\textbf{Theoretical Analysis.} We'll theoretically show  that minimizing $-\mathcal{L}_{\mathrm{VD}}$ encourages view generator to produce fair augmented views. 
		\begin{theorem}
			Assume the discriminator loss for each sample is bounded, then we show that 
			minimizing $-\mathcal{L}_{\mathrm{VD}}$ is equivalent to optimizing an upper bound on the group fairness $\Phi$ in Eq. (\ref{phi}).
			\label{th1}
		\end{theorem}
		\begin{proof}
			Let $\tilde{\bm{s}}$ be the predicted score of attribute information. We assume that the discriminator loss for each sample is bounded, meaning there exists a constant $\delta$ so that $\left|\bm{s}_i \log \tilde{\bm{s}}_i+\left(1-\bm{s}_i\right) \log \left(1-\tilde{\bm{s}}_i\right)\right| \leq \delta$. Together with the concavity of log(·) function and Jensen’s inequality, for any $\tilde{\bm{s}}_i \in \left[0,1\right]$ with $| \text{log}(\tilde{\bm{s}}_i)| \leq \delta$, we have:
			\begin{equation}
				\log (\tilde{\bm{s}}_i) \geq \frac{\delta}{1-e^{-\delta}} \cdot (\tilde{\bm{s}}_i-1).
			\end{equation}
		
			Then, we suppose the proportion of users in $\mathcal{U}_{\bm{s}=0}$ is $r_{0}$, and the proportion in $\mathcal{U}{\bm{s}=1}$ is $r_{1}$, with $r = \max(r_0, r_1)$.	
			Next, we derive group fairness $\Phi = \mathbb{E}_{i \sim \mathcal{U}_{\bm{s}=0}}\left[  \mathcal{F}_{i} \right]-\mathbb{E}_{j \sim \mathcal{U}_{\bm{s}=1}}\left[ \mathcal{F}_j \right]$, where $\mathcal{F}_{i}$ is the recommendation quality (e.g., recall, ndcg) of user $i$. Besides, for clarity of presentation, we introduce an intermediate variable $t = \frac{\delta}{1-e^{-\delta}}$. Hence, Eq. (\ref{dis}) can be rewritten as:
			\begin{equation}
				\begin{aligned}
					- \mathcal{L}_{\mathrm{VD}} 
					& =\frac{1}{|\mathcal{U}|} \sum_{i \in \mathcal{U} }\left[\bm{s}_i \log \tilde{\bm{s}}_i+\left(1-\bm{s}_i\right) \log \left(1-\tilde{\bm{s}}_i\right)\right] \\
					&=r~\mathbb{E}_{i \sim \mathcal{U}_{\bm{s}=0}}\left[\log \left(\tilde{\bm{s}}_i\right)\right]+r_{1}~\mathbb{E}_{i \sim \mathcal{U}_{\bm{s}=1}}\left[\log \left(1 - \tilde{\bm{s}}_i\right)\right] \\
					&\geq r\left[\mathbb{E}_{i \sim \mathcal{U}_{\mathbf{s}=0}}\left[\log \left(\hat{\bm{s}}_i\right)\right]+\mathbb{E}_{i \sim \mathcal{U}_{\bm{s}=1}}\left[\log \left(1-\tilde{\bm{s}}_i\right)\right]\right] \\
					&\geq r \left[\mathbb{E}_{i \sim \mathcal{U}_{\bm{s}=0}}\left[  (\tilde{\bm{s}}_i-1) \cdot t\right]-\mathbb{E}_{i \sim \mathcal{U}_{\bm{s}=1}}\left[ (\tilde{\bm{s}}_i) \cdot t \right]\right] \\
					&= (r \cdot t)  \left[\mathbb{E}_{i \sim \mathcal{U}_{\bm{s}=0}}\left[  (\tilde{\bm{s}}_i-1) \right]-\mathbb{E}_{i \sim \mathcal{U}_{\bm{s}=1}}\left[ \tilde{\bm{s}}_i \right]\right] \\
					&\stackrel{a}{=}(r \cdot t) \left[\mathbb{E}_{i \sim \mathcal{U}_{\bm{s}=0}}\left[  \hat{\bm{s}}_i \right]-\mathbb{E}_{i \sim \mathcal{U}_{\bm{s}=1}}\left[ \tilde{\bm{s}}_i \right]\right] - (r \cdot t)  \\ 
					&\stackrel{b}{=} (r \cdot t)~\Phi - (r \cdot t), \quad \text{where } r \cdot t > 0.
				\end{aligned}
				\label{fair}
			\end{equation}
			
			We find the essence of group fairness $\Phi$ lies in reducing the disparity in recommendation quality between two distinct user groups. In other words, this involves minimizing the differences in user embeddings across these groups.
			Remarkably, this aligns perfectly with the intention of step $a$ in Eq.~(\ref{fair}), thereby establishing the foundation for step $b$. Indeed, Eq.~(\ref{fair}) shows $-\mathcal{L}_{\mathrm{VD}} $ is an upper for group fairness $\Phi$, that is to say maximizing the view discriminator loss is equivalent to minimizing the unfairness of the view generator. Therefore, the above statement in Theorem~\ref{th1} holds.
		\end{proof}
		\subsection{Learning informative View Generator}
		In theory, optimizing Eq.~(\ref{dis}) can indeed help the view generator to reduce data unfairness and generate fair node embedding. However, if not properly controlled, it might adopt extreme strategies.
		For instance, the generators could consistently generate node representations with all features set to zero~\cite{ling2023learning}. 
		While this would eliminate unfairness entirely, it would also remove valuable information from the original data, making the generated data essentially useless.
		To ensure that the view generator model retains its informativeness, we have also introduced two important loss functions to preserve the most essential information from the original graph.
		
		\noindent
		\underline{\textbf{Task-Guide BPR Loss.}}
		While view generator may generate fair embeddings from different perspectives, there may not be optimization signals to align these generated views with the main task.
		Thus, we employ the BPR loss, as shown in Eq.~(\ref{bprloss}), to individually optimize each view generator. For the first recognition model:
		\begin{equation}
			\mathcal{L}^{rec}_{\mathrm{BPR}} (u, v^{+}, v^{-}) = -\log \sigma\left(y^{r}\left(u, v^{+}\right)-y^{r}\left(u, v^{-}\right)\right),
			\label{task}
		\end{equation}
		where $\left(u, v^{+}\right)$ indicates the positive interaction in $\mathcal{G}_{1}$, $\left(u, v^{-}\right)$ is a randomly chosen negative interaction, $y^{r}$ denotes the interaction score from the first recognition model. Then, the $\mathcal{L}^{gen}_{\mathrm{BPR}}$ in the second generative model can be derived in a similar manner.

		\noindent
		\underline{\textbf{Contrastive Learning Loss.}}
		In line with existing self-supervised learning paradigms~\cite{huang2023adversarial,zhang2023automated}, we pull close the representations for the same node in two different views, and push away the embeddings for different entities in the two views. Based on these two augmented embeddings, the contrastive loss is formally defined by:
		\begin{equation}
			\mathcal{L}_{\mathrm{NCE}}=\sum_{i \in \mathcal{U} \cap \mathcal{V}}-\log \frac{\exp \left(\text{cos}\left(\mathbf{H}_1(i), \mathbf{H}_{2}(i)\right) / \tau\right)}{\sum_{i \neq j} \exp \left(\text{cos}\left(\mathbf{H}_1(i), \mathbf{H}_{2}(j) / \tau\right)\right.},
			\label{nce}
		\end{equation}
		where $\text{cos}(\cdot)$ represents the cosine similarity function and $\tau$ is the temperature hyper-parameter. $\mathbf{H}(i) \in \mathbb{R}^d$ denotes the embedding vector in the $i$-th row. In practice, this contrastive loss $\mathcal{L}_{\mathrm{NCE}}$ not only ensures the alignment of  two views but also enhances the view generator's capability to produce more informative representation.

		\noindent
		\textbf{Theoretical Analysis.} 
		Here, we'll show that minimizing contrastive loss can help view generator create more informative views.
		\begin{theorem}
			Given the original graph $\mathcal{G}$ and its generative views $\mathcal{G}_{1}$ and $\mathcal{G}_{2}$, along with their corresponding embeddings ${\mathbf{H}_{1}}$ and ${\mathbf{H}_{2}}$.
			Contrastive learning objective is a lower bound of mutual information between $\mathcal{G}$ and generative views $\mathcal{G}_{1}$, $\mathcal{G}_{2}$.
			\label{th2}
		\end{theorem}
		\begin{proof}
			First, for two random variables $X$ and $Y$, we define the mutual information $	\mathcal{I}(X,Y)$ as:
			\begin{equation}
				\mathcal{I}(X ; Y)=\sum_{x \in X, y \in Y} p(x, y) \log \frac{p(x, y)}{p(x) p(y)},
		\end{equation}
			where $p(x,y)$ is the joint distribution of $X$ and $Y$, and $p(x)$ and $p(y)$ are marginal distributions. 
			Then, we consider a general InfoNCE contrastive loss, it can be expressed as:
			\begin{equation}
				\mathcal{L}_{\text{NCE}}=-\mathbb{E}_{p(x, y)}\left[\log \frac{e^{g(x, y)}}{\sum_{y^{-}} e^{g\left(x, y^{-}\right)}}\right],
			\end{equation}
			where $(x,y)$ denotes positive sample pair, and $(x,y^{-})$ is negative sample pair.  $g(\cdot)$ is a scoring function designed to assign higher scores to positive samples. 
			Following previous works~\cite{zhu2020deep,ling2023learning}, combining jensen's inequality and logarithms and inequalities, we can derive:
			\begin{equation}\mathcal{I}(X ; Y) \geq \log \left(1+e^{-\mathcal{L}_{\text{NCE}}}\right).
			\end{equation}
		
		This shows that minimizing the contrast loss is equivalent to maximizing the lower bound of mutual information of $X$ and $Y$\footnote{Note that the proof here is only an overview, and detailed mathematical proof requires more rigorous mathematical derivation.}.
			Thus, in this work, we can state:
			\begin{equation}
				-\mathcal{L}_{\text{NCE}} \leq \mathcal{I}\left(\mathbf{H}_{1} \vcenter{\hbox{;}}{\mathbf{H}_{2}}\right).
			\end{equation}
			
			In information theory~\cite{gallager1968information}, the data processing inequality states that for a Markov chain consisting of a sequence of random variables $\mathbf{X}\rightarrow \mathbf{Y}\rightarrow\mathbf{Z}$, it must adhere to the inequality $\mathcal{I}(\mathbf{X}\vcenter{\hbox{;}}\mathbf{Y}) \geq \mathcal{I}(\mathbf{X}\vcenter{\hbox{;}}\mathbf{Z})$. 
			This inequality ensures a logical and concise flow of information. 
			In our proposed FairDgcl, we can observe that $\mathcal{G}$, $\mathcal{G}_{1}$ and $\mathcal{G}_{2}$ follow the relationship $\mathcal{G}_{1} \leftarrow \mathcal{G} \rightarrow \mathcal{G}_{2}$.
			Considering that $\mathcal{G}_{1}$ and $\mathcal{G}_{2}$ are conditionally independent when $\mathcal{G}$ is observed, this scenario is Markov equivalent to the chain $\mathcal{G}_{1} \rightarrow \mathcal{G} \rightarrow \mathcal{G}_{2}$. Therefore, we have $\mathcal{I}\left(\mathcal{G}_{2}\vcenter{\hbox{;}} \mathcal{G}\right) \geq \mathcal{I}\left(\mathcal{G}_{1}\vcenter{\hbox{;}} \mathcal{G}_{2}\right)$.
			In the given context, we also observe that the mutual information between multiple views and the original graph is not less than the mutual information between a single view and the original graph~\cite{zhu2020deep}.
			From this, we can deduce the following inequality: $\mathcal{I}(\mathcal{G}\vcenter{\hbox{;}}\mathcal{G}_{1}, \mathcal{G}_{2}) \geq \mathcal{I}(\mathcal{G}\vcenter{\hbox{;}} \mathcal{G}_{2}) $.
			Finally, by combining the relationship $ \mathbf{H}_{1}\leftarrow \mathcal{G}_{1} \leftarrow \mathcal{G} \rightarrow \mathcal{G}_{2} \rightarrow \mathbf{H}_{2}$ and previous equations, we can easily get the following inequality:
			\begin{equation}
				-\mathcal{L}_{\text{NCE}} \leq  \mathcal{I}\left(\mathbf{H}_{1} \vcenter{\hbox{;}}{\mathbf{H}_{2}}\right) \leq
				\mathcal{I}\left(\mathcal{G}_{1} \vcenter{\hbox{;}}{\mathcal{G}_{2}}\right)  \leq
				\mathcal{I}\left(\mathcal{G}\vcenter{\hbox{;}}{\mathcal{G}_{1}},{\mathcal{G}_{2}}\right),
			\end{equation}
			thus, the statement presented in Theorem~\ref{th2} remains valid.
		\end{proof}
		\begin{algorithm}[t]
			\setstretch{1}
			\caption{Training Details of FairDgcl}
			\KwInput{input graph $\mathcal{G}$, learning rate $\gamma$, parameters $\theta_{g}$, $\theta_{f}$, and $\theta_{d}$.}
			\While{Converged == False}{
				\textbf{\underline{Generate Augmented Representation:}}\\
				Generate augmented views $\mathcal{G}_{1}$ and $\mathcal{G}_{2}$ ;\\
				Obtain encoding views ${\mathbf{H}}_{1}$ and ${\mathbf{H}}_{2}$;\\
				\textbf{\underline{Calculate all the loss function $\mathcal{L}$:}}\\
				Calcuate $\mathcal{L}_{\mathrm{BPR}}$, $\mathcal{L}_{\mathrm{VG}}$, and $\mathcal{L}_{\mathrm{VD}}$; \\
				$\mathcal{L} = \mathcal{L}_{\mathrm{BPR}}+ \mathcal{L}_{\mathrm{VG}}- \beta\mathcal{L}_{\mathrm{VD}}$;\\ 
				\textbf{\underline{Optimize parameter $\theta_{g}$, $\theta_{f}$:}}\\
				minimize $\mathcal{L}$ using stochastic gradient descent;\\
				Update $\theta_{g} = \theta_{g}-\gamma \nabla_{\theta_{g}} \mathcal{L}$ ;\\
				Update $\theta_{f} = \theta_{f}-\gamma \nabla_{\theta_{f}} \mathcal{L}$ ;\\
				\textbf{\underline{ Optimize parameter $\theta_{d}$:}}\\
				maximize $\mathcal{L}$ using stochastic gradient descent;\\
				Update $\theta_{d} = \theta_{d}-\gamma \nabla_{\theta_{d}} \mathcal{L}$.\\ }
			\label{algorithm}
		\end{algorithm}

		\subsection{Overall Optimization}
		The overall training process of our proposed model comprises three components: 
		the main task's BPR loss function, the loss function for the view generator, and the loss function for the view discriminator. To be specific, the view generator consists of two contrastive views, with its loss function $	\mathcal{L}_{\mathrm{VG}}$ can be expressed as follows:
		\begin{equation}
			\begin{aligned}
				\mathcal{L}_{\mathrm{VG}}=  \mathcal{L}_{\mathrm{VGAE}} +	\mathcal{L}_{\mathrm{BPR}}^{rec} + 	\mathcal{L}_{\mathrm{BPR}}^{gen} + \alpha\mathcal{L}_{\text{NCE}},
			\end{aligned}
		\end{equation}
		where $\alpha$ is hyper-parameters that control the strength of contrastive section.
		Finally, by combining $	\mathcal{L}_{\mathrm{VD}}$ and 	$\mathcal{L}_{\mathrm{BPR}}$,
		these three components adhere to the mini-max adversarial optimization  as follows:
		\begin{equation}
			\min _{\theta_{g}, \theta_{f}} \max _{\theta_{d}} \mathcal{L} = 	\min _{\theta_{g}, \theta_{f}} \max _{\theta_{d}} \mathcal{L}_{\mathrm{BPR}}+ \mathcal{L}_{\mathrm{VG}}-\beta \mathcal{L}_{\mathrm{VD}},
		\end{equation}
		where $\theta_{f}$ is the parameter set of graph encoder\footnote{It is worth noting that all the parameters of graph encoder are shared in FairDgcl.}. $\theta_{g}$ and $\theta_{d}$ denote the parameter set of view generator and view discriminator, respectively. $\beta$ regulates the impact of view discriminator.
		The parameters of $\theta_{g}$, $\theta_{f}$, and $\theta_{d}$ are optimized together using a min-max optimization procedure.
		During each training step, we begin by minimizing $\mathcal{L}$ through updating $\theta_{g}$ and $\theta_{f}$, with $\theta_{d}$ held constant. Subsequently, we maximize $\mathcal{L}$ by updating $\theta_{d}$ while keeping $\theta_{g}$ and $\theta_{f}$ fixed.
		The overall training details are presented in Algorithm~1.
		
		\vspace{3pt}
\noindent
		\textbf{Comparison with Existing Methods.} 
		Learning fair user representation in recommender system has become a vital research focus, with methodologies ranging from model optimization to the use of auxiliary interaction data. Our research aligns with the latter, emphasizing data augmentation as a pivotal strategy.
		Compared with existing methods~\cite{ying2023camus,chen2023fairgap,ling2023learning}, our approach has two major technical contributions.
        Firstly, we unveil a graph adversarial contrastive learning framework designed to rectify fairness issues in recommendations. This framework consists of a learned view generator, which is responsible for developing fair augmentation strategies and generating equitable representations, and a view discriminator, which evaluates the fairness of these representations. These components collaboratively work in an adversarial manner to enhance the model and reduce bias within the input graph, effectively addressing the two challenges mentioned in Section 1.
        Secondly, our FairDgcl method pioneers an automated graph augmentation process tailored specifically for group-based fairness in recommender systems. This method is not only rooted in theoretical principles that ensure fairness and informativeness but has also undergone empirical validation through extensive testing, showcasing its effectiveness.

		\begin{table}[t]
			\caption{Statistics of Four Experimental Datasets.}
			\vspace{-10pt}
			\resizebox{0.99\linewidth}{!}{
				\begin{tabular}{c|c|c|c|c}
					\hline 
					& ML-1M & ML-100K&Last.FM & \multicolumn{1}{l}{IJCAI-2015} \\ \hline  \hline
					\# Users & 6,040      &943    & 9,900   &  7,526                     \\
					\# Items &   3,952    &1,682    & 31,834   &    4,506                   \\
					\# Interactions  & 1,000,209      &100,000    & 285,201   & 274,431                      \\
					Sparsity     &  95.81\%     &93.70\%    & 99.91\%   & 99.19\%                      \\
					Domain &   Movies    & Movies    & Music    &    E-commerce        \\ \hline 
			\end{tabular}}
			\label{data}
		\end{table}
		\begin{table*}[h]
			\centering
			\caption{Comparison of overall performance in terms of recommendation accuracy and group fairness across four datasets with varying Top-$K$ values.  The highest scores are highlighted in bold. $*$ indicates the improvement of FairDgcl over the Graphair is significant at the level of 0.05, and ``$\uparrow$'' (``$\downarrow$'') indicates larger (smaller) values are better.
			}
			\renewcommand\arraystretch{1}
			\resizebox{0.999\linewidth}{!}{
				\begin{tabular}{c|c|cccc|cccc|cccc}
					\hline
					\multirow{2}{*}{Dataset} & \multirow{2}{*}{Method} & \multicolumn{4}{c|}{$K$=10}         & \multicolumn{4}{c|}{$K$=20}         & \multicolumn{4}{c}{$K$=30}          \\ \cline{3-14} 
					&                         & Recall~$\uparrow$& NDCG~$\uparrow$   & $\Phi_R$~$\downarrow$    & $\Phi_N$~$\downarrow$    & Recall~$\uparrow$ & NDCG~$\uparrow$   & $\Phi_R$~$\downarrow$    & $\Phi_N$~$\downarrow$    & Recall~$\uparrow$ & NDCG~$\uparrow$   & $\Phi_R$~$\downarrow$    & $\Phi_N$~$\downarrow$    \\ \hline \hline
					\multirow{7}{*}{ML-1M}   & LightGCN                & 0.1194 & 0.1945 & 0.0099 & 0.0315 & 0.1867 & 0.2027 & 0.0156 & 0.0304 & 0.2401 & 0.2157 & 0.0219 & 0.0295 \\
					& FairDrop                & 0.1144 & 0.1891 & 0.0092 & 0.0313 & 0.1791 & 0.1945 & 0.0152 & 0.0309 & 0.2395 & 0.2149 & 0.0192 & 0.0297 \\
					& FairGO                 & 0.1201        & 0.1961       &  0.0114      & 0.0325       &  0.1868      & 0.2027        & 0.0143       & 0.0293       & 0.2410       &  0.2163      &  0.0190      & 0.0293       \\
					& CAMUS                   & 0.1162       &  0.1918      & 0.0089       & 0.0319        & 0.1828        & 0.1994        &  0.0134      &  0.0295      & 0.2381       &   0.2141     &  0.0181      &  0.0291      \\
					& FairGap                     & 0.1178 & 0.1933 & 0.0087 & 0.0288 & 0.1846 & 0.2010 & 0.0117 & 0.0254 & 0.2382 & 0.2147 & 0.0178 & 0.0264 \\
					& Graphair                &  0.1435      & 0.2092       & 0.0010       &  0.0238      &  0.2275      &  0.2280      & 0.0073       & 0.0217       &  \textbf{0.2993}      &  \textbf{0.2553}      & 0.0134       & 0.0247       \\
					& \textbf{FairDgcl (Ours)}    &  \textbf{0.1531}$^{*}$      &  \textbf{0.2131}$^{*}$      & \textbf{0.0009}$^{*}$       & \textbf{0.0176}$^{*}$       & \textbf{0.2393}$^{*}$       &  \textbf{0.2295}$^{*}$      & \textbf{0.0061}$^{*}$       &  \textbf{0.0169}$^{*}$      &  0.2980     & 0.2474       &   \textbf{0.0100}$^{*}$     & \textbf{0.0164}$^{*}$      \\ \hline \hline
					\multirow{7}{*}{ML-100K}   
					&LightGCN	 &0.2300	 &0.2654	 &0.0302	 &0.0102	 &0.3512	 &0.2974	 &0.0163	 &0.0076	 &\textbf{0.4324}	 &0.3254	 &0.0142	 &0.0074        \\
					&FairDrop	&0.2235	&0.2631	&0.0189	&0.0048	&0.3398	&0.2913	&0.0272	&0.0061	&0.4248	&0.3197	&0.0238	&0.0049 \\
					&FairGO	&0.2297	&0.2695	&0.0274	&0.0050	&0.3482	&\textbf{0.2990}	&0.0307	&0.0062	&0.4277	&0.3250	&0.0341	&0.0107 \\
					&CAMUS	&0.2253	&0.2657	&0.0248	&0.0470	&0.3427	&0.2944	&0.0355	&0.0120	&0.4245	&0.3211	&0.0240	&0.0079 \\
					&FairGap	&0.2271	&0.2706	&0.0200	&0.0052&0.3457	&0.2953	&0.0222	&0.0074&0.4245	&0.3247	&0.0280&0.0103 \\
					&Graphair	&0.2299	&0.2658	&0.0204	&0.0067	&\textbf{0.3514}	&0.2982	&0.0200	&0.0019	&0.4264	&0.3234	&0.0161	&0.0052 \\
					&\textbf{FairDgcl (Ours)}&\textbf{0.2395}$^{*}$	&\textbf{0.2743}$^{*}$	&\textbf{0.0163}$^{*}$	&\textbf{0.0043}$^{*}$	&0.3415	&0.2924	&\textbf{0.0122}$^{*}$	&\textbf{0.0011}$^{*}$	&0.4299&\textbf{0.3280}$^{*}$	&\textbf{0.0139}$^{*}$	&\textbf{0.0035}$^{*}$     \\ \hline \hline
					\multirow{7}{*}{Last.FM}  
					&LightGCN	&0.0955	&0.0731	&0.0066	&0.0085	&0.1477	&0.0927	&0.0137	&0.0112	&0.1887	&0.1056	&0.0124	&0.0109       \\
					&FairDrop	&0.0953	&0.0728	&0.0071	&0.0087	&0.1472	&0.0924	&0.0124	&0.0106	&0.1884	&0.1057	&0.0132	&0.0110      \\
					&FairGO	&0.0936	&0.0716	&0.0063	&0.0070	&0.1448	&0.0909	&0.0085	&0.0074	&0.1862	&0.1037	&0.0077	&0.0075 \\
					&CAMUS	&0.0955	&0.0732	&0.0085	&0.0093	&0.1478	&0.0928	&0.0133	&0.0111	&0.1888	&0.1059	&0.0137	&0.0111\\
					&FairGap	&0.0941	&0.0723	&0.0074	&0.0087	&0.1460	&0.0918&0.0117	&0.0101	&0.1854	&0.10544&0.0143	&0.0111\\
					&Graphair	&0.1153	&0.0912	&0.0075	&0.0085	&0.1754	&0.1136	&0.0074	&0.0088	&0.2219	&0.1279	&0.0046	&0.0073\\
					&\textbf{FairDgcl (Ours)}	&\textbf{0.1312}$^{*}$&\textbf{0.1053}$^{*}$	&\textbf{0.0042}$^{*}$	&\textbf{0.0058}$^{*}$	&\textbf{0.1954}$^{*}$	&\textbf{0.1295}$^{*}$	&\textbf{0.0030}$^{*}$&\textbf{0.0054}$^{*}$	&\textbf{0.2442}$^{*}$ &\textbf{0.1449}$^{*}$	&\textbf{0.0042}$^{*}$	&\textbf{0.0056}$^{*}$\\
					\hline \hline
					\multirow{7}{*}{IJCAI-2015}  
					&LightGCN	&0.0821	&0.0676	&0.0065	&0.0006	&0.1263	&0.0838	&0.0078	&0.0009	&0.1629	&0.0958	&0.0145	&0.0009 \\
					&FairDrop	&0.0826	&0.0685	&0.0050	&0.0014	&0.1267	&0.0843	&0.0079	&0.0017	&0.1634	&0.0963	&0.1222	&0.0006\\
					&FairGO	&0.0813	&0.0673	&0.0053	&0.0008	&0.1262	&0.0836	&0.0095	&0.0005	&0.1617	&0.0956	&0.0135	&\textbf{0.0004}\\
					&CAMUS	&0.0821	&0.0676	&0.0063	&0.0012	&0.1278	&0.0841	&0.0092	&\textbf{0.0003}	&0.1632	&0.0959	&0.0156	&0.0008\\
					&FairGap	&0.0803	&0.0671	&0.0076	&0.0008&0.1256	&0.0837&0.0098	&0.0010	&0.1614	&0.0954&0.0135	&0.0014\\
					&Graphair	&0.0845	&0.0697&0.0048&0.0013	&0.1290	&0.0848	&0.0098	&0.0009	&0.1643	&0.0967	&0.0171	&0.0019\\
					&\textbf{FairDgcl (Ours)} &\textbf{0.0859}$^{*}$	&\textbf{0.0705}$^{*}$	&\textbf{0.0030}$^{*}$ &\textbf{0.0007}$^{*}$	&\textbf{0.1324}$^{*}$	&\textbf{0.0878}$^{*}$	&\textbf{0.0073}$^{*}$	&0.0023	&\textbf{0.1693}$^{*}$	&\textbf{0.1000}$^{*}$	&\textbf{0.0128}$^{*}$	&0.0018 \\ \hline 
			\end{tabular}}
			\label{main}
		\end{table*}
		\section{EXPERIMENTS}
		In this section, we present empirical results to illustrate the effectiveness of our proposed FairDgcl framework. 
		The experiments are designed to address the following four research questions:
		\begin{itemize}[leftmargin=*,labelindent=0pt,topsep=2pt,itemsep=1pt]
			\item \textbf{RQ1:} How does the performance of our proposed model compare to that of various baselines?
			\item \textbf{RQ2:} How do the key components of the proposed FairDgcl framework affect the overall performance?
			\item \textbf{RQ3:} How sensitive is the performance of FairDgcl to variations in its hyper-parameters?
			\item \textbf{RQ4:}  Do noticeable variations exist in the representation of distinct user demographics within our model?
		\end{itemize}
		
		\subsection{Experimental Settings}
		\subsubsection{\textbf{Dataset Description.}} We conduct experiments on four benchmark datasets: Movielens-1M\footnote{https://group.org/dataset/movielens} (ML-1M), Movielens-100K\footnote{https://grouplens.org/datasets/movielens/100k} (ML-100K), Last.FM\footnote{https://millionsongdataset.com/lastfm}, and IJCAI-2015\footnote{https://tianchi.aliyun.com/dataset/42}, which are widely used in fairness-aware recommendations ~\cite{wu2021learning,chen2023fairgap,chen2023improving}. The dataset statistics are summarized in Table~\ref{data}. ML-1M and ML-100K~\cite{harper2015movielens}, are two popular movie datasets for recommender systems, which encompass comprehensive records of user-item interactions and user profile information. Last.FM~\cite{celma2009music} is a music dataset collected from Last.fm music website that contains users' ratings to artists.
		IJCAI-2015~\cite{liu2022dual} is sourced from Tmall.com and focuses on e-commerce activities. It contains anonymized users' shopping logs from the past 6 months leading up to and including the ``Double 11’’ day. We use gender as the demographic sensitive attribute for evaluating fairness performance.
		For each dataset, we perform a random split of the interactions into three distinct sets: training (0.8), validation (0.1), and test (0.1). 

		\subsubsection{\textbf{Compared Models.}} For performance evaluation, we compare FairDgcl with several representative baselines that have been designed to achieve fairness in recommender systems:
		\begin{itemize}[leftmargin=*,labelindent=0pt,topsep=2pt,itemsep=1pt]
			\item \textbf{LightGCN}~\cite{he2020lightgcn} utilizes neighborhood information in the user-item graph through a layer-wise propagation scheme, which involves only linear transformations and element-wise additions.
			\item \textbf{FairDrop}~\cite{spinelli2021fairdrop} is a graph editing method that selectively obscures edges between nodes sharing similar attributes.
			\item \textbf{FairGO}~\cite{wu2021learning} utilizes a generative adversarial network to remove sensitive information from user representation, relying on the user-item bipartite graph as a foundation.
			\item \textbf{CAMUS}~\cite{ying2023camus} introduces a counterfactual data augmentation strategy tailored for non-mainstream users, designed to not only improve the performance of these users but also to reduce disparities across different user groups.
			\item \textbf{FairGap}~\cite{chen2023fairgap} is a reweighting approach inspired by counterfactual thinking, aimed at generating a fair graph and learning fair node representations.
			\item \textbf{Graphair}~\cite{ling2023learning} is a learnable data augmentation approach that improve graph fairness for node classification task, achieved through perturbing edges and masking node features in a adversarial learning framework.
			
		\end{itemize}
		\subsubsection{\textbf{Evaluation Protocols.}}
		As our primary focus on trade-off between recommendation fairness and accuracy, we need to evaluate two aspects and report the trade-off results. 
		We use the all-rank evaluation protocol, which tests and ranks both the positive items and all non-interacted items for each test user in the test set.
		Firstly,
		for recommendation accuracy evaluation, we adopt commonly-used Recall$@K$
		and Normalized Discounted Cumulative Gain (NDCG)$@K$ as metrics. 
		Larger values of Recall and NDCG indicate better recommendation accuracy performance. Then, we use group fairness $\Phi$ (in Eq.~(\ref{phi})) to evaluate fairness. $\Phi_{R}$ ($\Phi_{N}$) refers to the difference in terms of recall (ndcg) performance in different user groups.
		It's important to highlight that smaller values for $\Phi_{R}$ and $\Phi_{N}$ indicate better fairness results, 
		Following~\cite{chen2023improving}, we set the value of $K$ to $\{10,20,30\}$. 
		\subsubsection{\textbf{Implement Details.}}
		In this experiment, we use the PyTorch framework to implement FairDgcl. 
		For parameter inference, we employ the Adam~\cite{kingma2014adam} optimizer with a learning rate of 1$e^{-3}$.
		The default hidden dimensionality is set to 64, and the training epoch is set to 200.
		The hyper-parameters are determined based on grid search. In more detail,  the number of graph neural iterations is tuned from $\{1,2,3,4\}$, and the batch size are tuned in the ranges of $[512,1024,2048,4096]$. In the model training phase,  $\alpha$ and $\beta$  are respectively chosen from $\{0.0001,0.001,0.01,0.1,1 \}$.
		We fine-tune the parameters of all baseline methods and report the best results.
		
			\begin{figure*}[h]
			\centering
			\begin{subfigure}[t]{0.22\textwidth}
				\includegraphics[width=\textwidth]{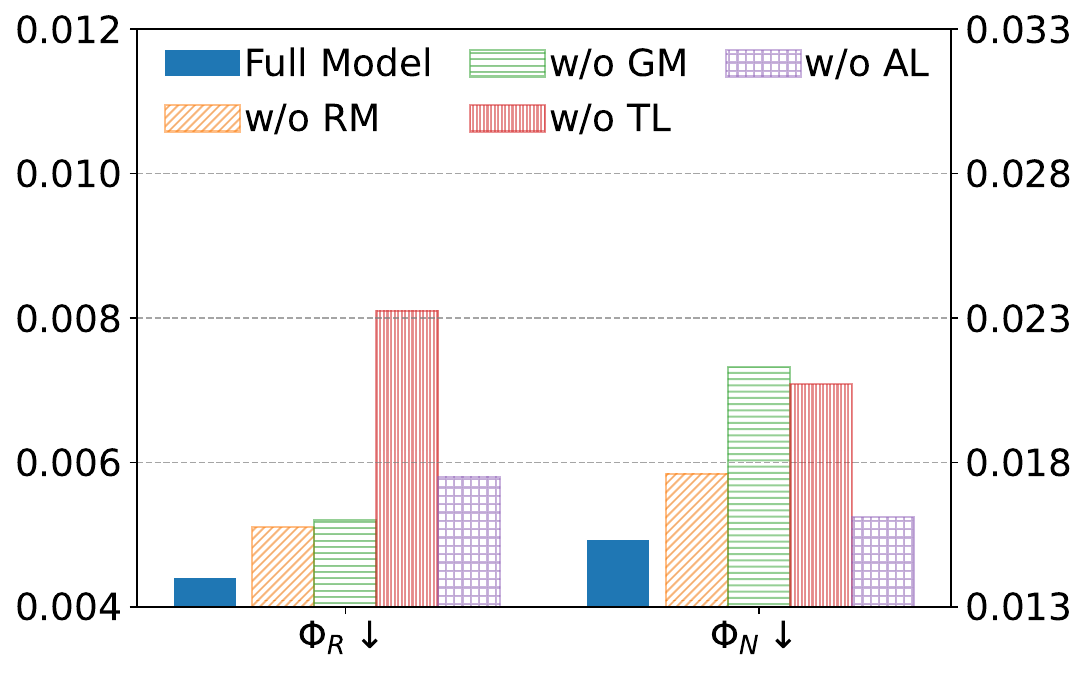}
			\end{subfigure}
			\begin{subfigure}[t]{0.22\textwidth}
				\includegraphics[width=\textwidth]{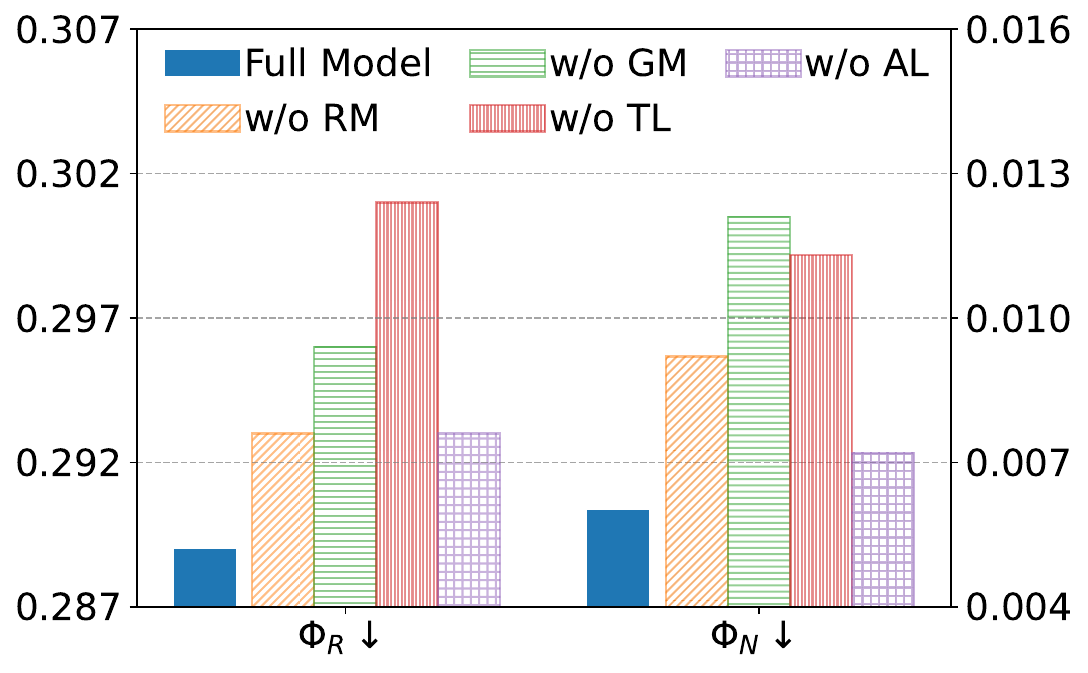}
			\end{subfigure}
			\begin{subfigure}[t]{0.22\textwidth}
				\includegraphics[width=\textwidth]{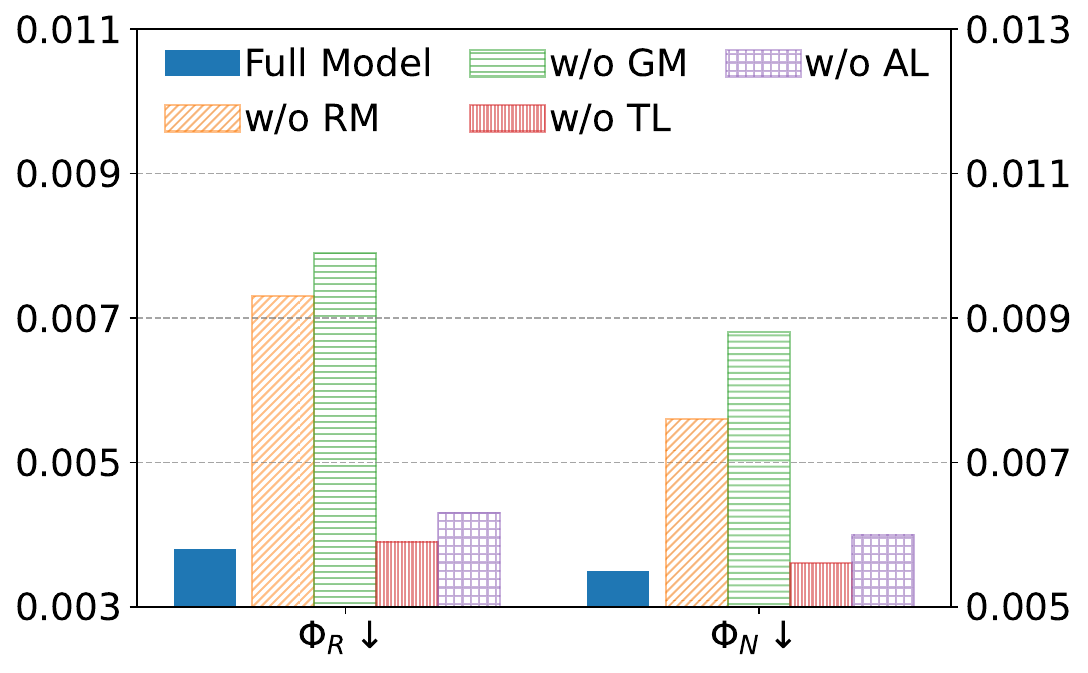}
			\end{subfigure}
			\begin{subfigure}[t]{0.22\textwidth}
				\includegraphics[width=\textwidth]{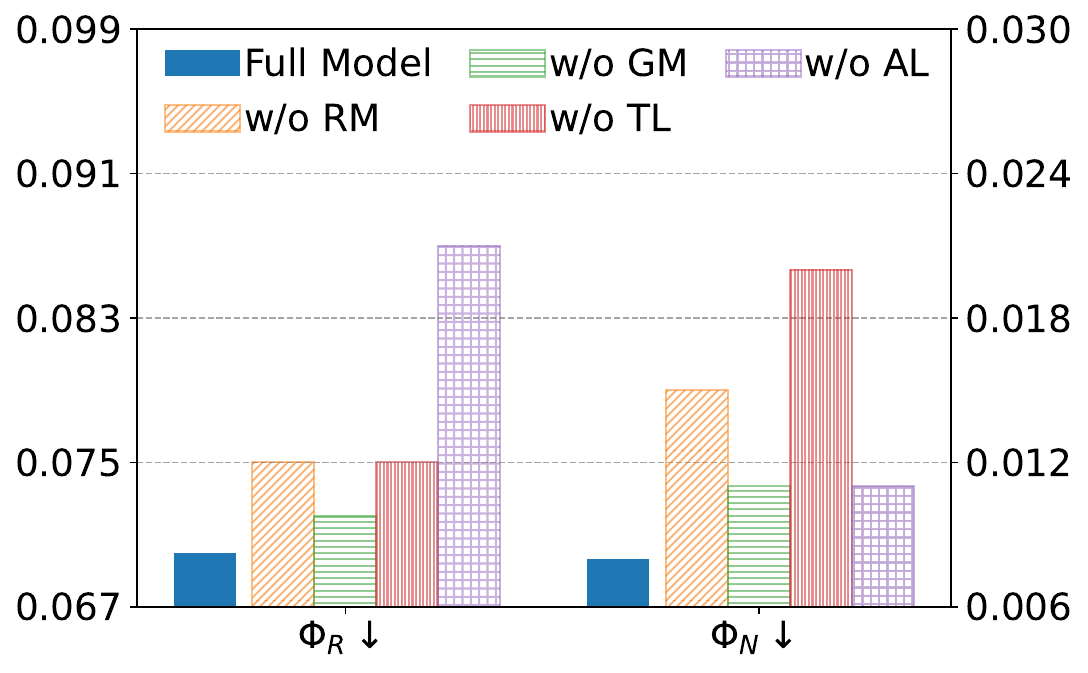}
			\end{subfigure}
			
			\begin{subfigure}[t]{0.22\textwidth}
				\includegraphics[width=\textwidth]{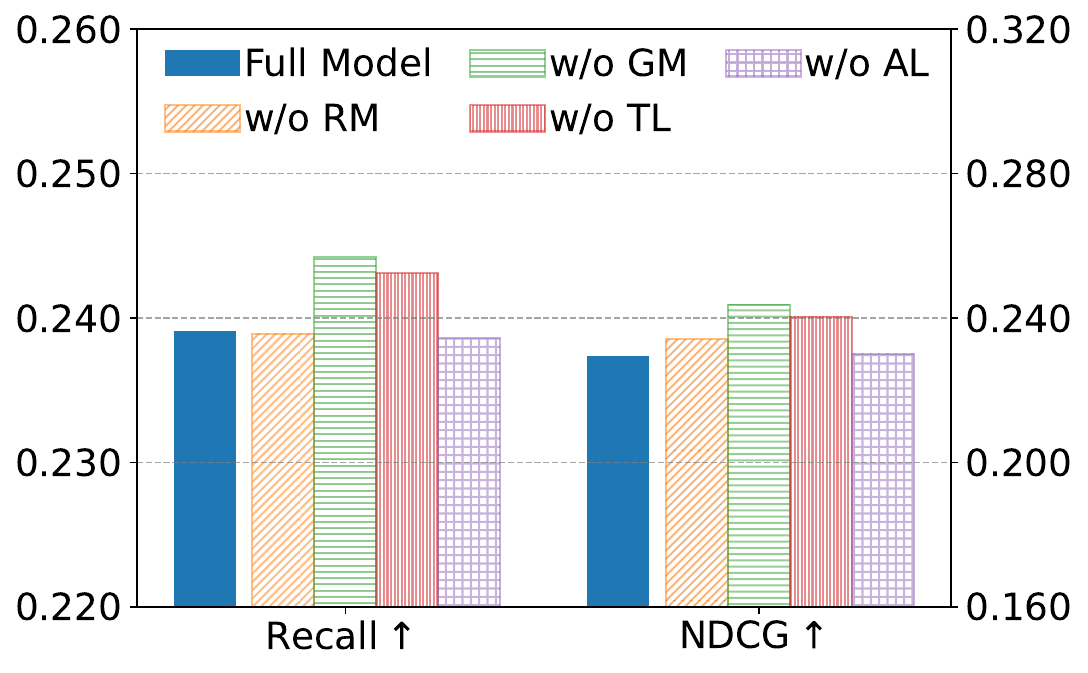}
				\caption{\normalfont ML-1M}
			\end{subfigure}
			\begin{subfigure}[t]{0.22\textwidth}
				\includegraphics[width=\textwidth]{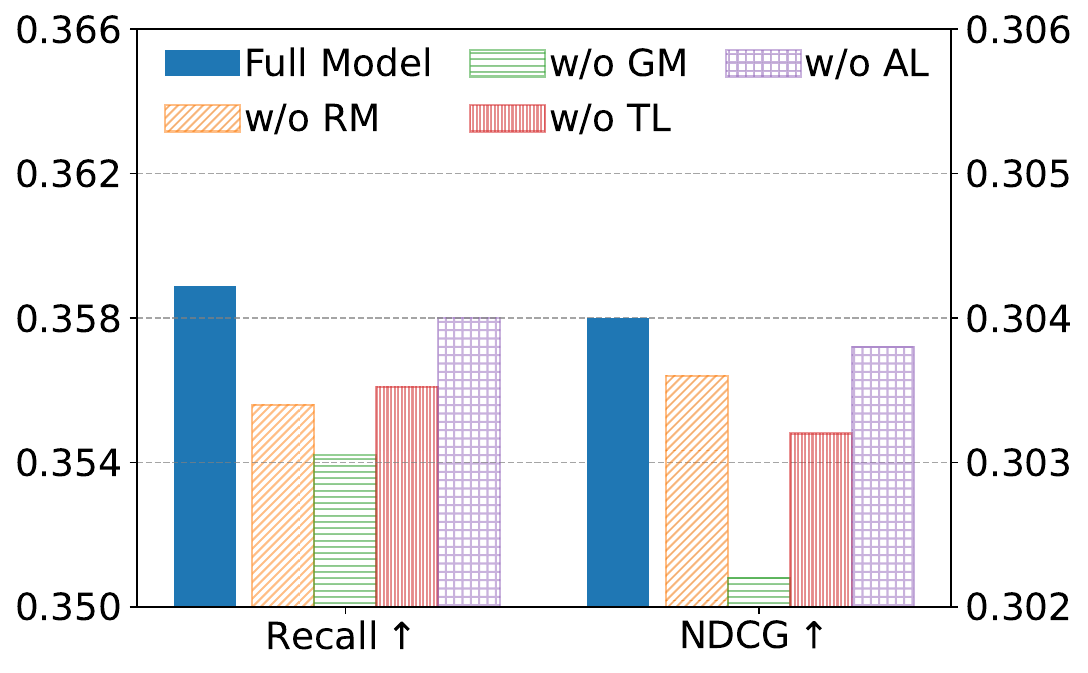}
				\caption{\normalfont ML-100K}
			\end{subfigure}
			\begin{subfigure}[t]{0.22\textwidth}
				\includegraphics[width=\textwidth]{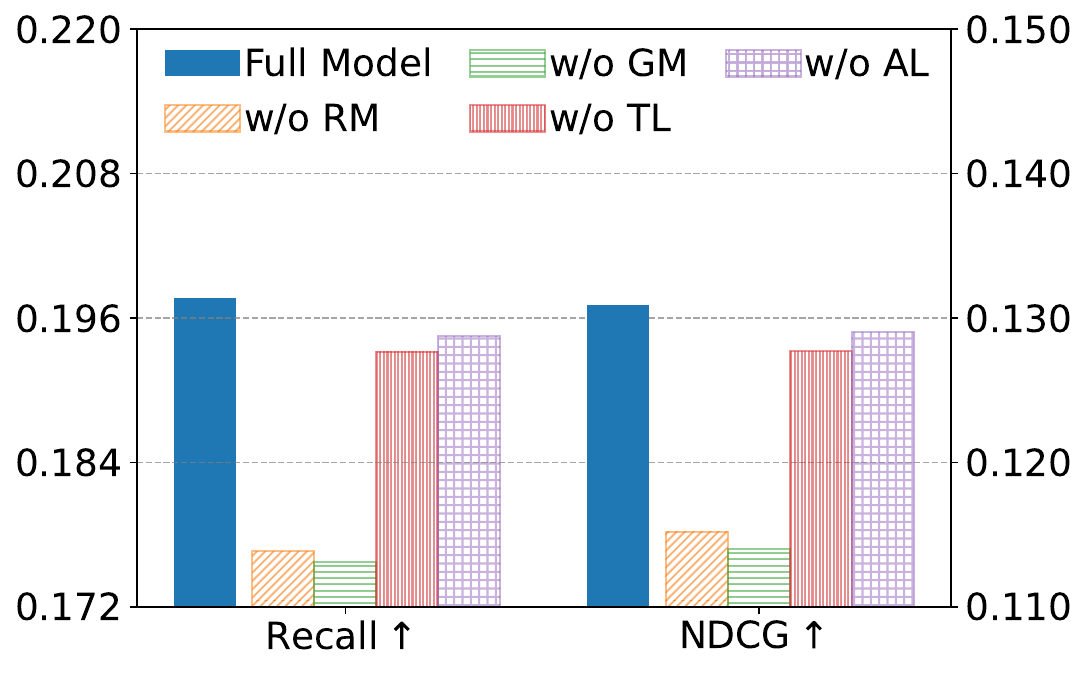}
				\caption{\normalfont Last.FM}
			\end{subfigure}
			\begin{subfigure}[t]{0.22\textwidth}
				\includegraphics[width=\textwidth]{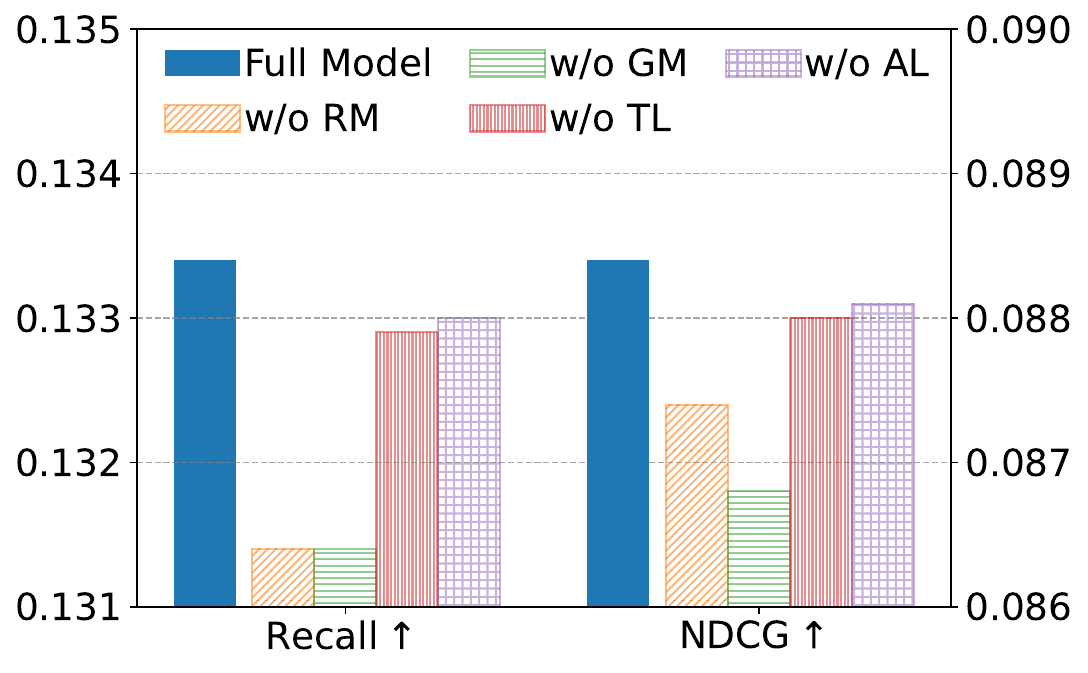}
				\caption{\normalfont IJCAI-2015}
			\end{subfigure}
			\caption{Comparing FairDgcl and its variants, all results are based on the Top-20 evaluation criteria.
				“w/o” means removal operation, “RM” stands for recognition model, “GM” for generative model, “TL” indicates the task-guide loss function (Eq.~(\ref{task})), and “AL” represents the adversarial loss (Eq.~(\ref{dis})). Note that ``$\uparrow$''  (``$\downarrow$'') indicates better (worse) performance. }
			\label{ablation}
		\end{figure*}
		\subsection{Experimental Results (RQ1)}
	
		
		
		We compare FairDgcl with different baselines on four real-world datasets to evaluate its effectiveness. To ensure a fair comparison, we try to maintain consistent recommendation accuracy across all baselines.
		The results are presented in Table~\ref{main} and our observations can be summarized as follows:
		
		Overall, our proposed FairDgcl model consistently outperforms others in terms of both group fairness scores and recommendation results across the majority of validation metrics. This not only validates the model's exceptional performance in ensuring fair recommendations but also highlights its remarkable recommendation quality, further confirming its effectiveness.
		These improvements can be attributed to two key factors. Firstly, 
		FairDgcl employs an adversarial learning framework, comprising a view generator and a view discriminator, to automatically generate fair view through a min-max game, thereby improving fairness.
		Secondly, the introduction of contrastive learning paradigm with informative contrastive views enhances the model's informativeness. These strategies collectively contribute to our model achieving superior fairness and accuracy.
		
		Among the fairness-aware baseline models, Graphair stands out as the top performer, just confirming the importance of dynamic data augmentation in fairness-aware recommender systems.
		However, in most cases, Graphair slightly falls short of our model's performance. This discrepancy may be attributed to Graphair's focus solely on edge deletion augmentation, which results in a somewhat coarser method design, overlooking the augmentation of other perspectives and consequently leading to suboptimal performance. In contrast, our model employs two different augmentation approaches based on the contrastive learning framework, effectively enhancing model's performance.
		As for the CAMUS model, it typically yields improved fairness outcomes, but its recommendation accuracy tends to be unstable.
		The same holds true for FairDrop, as they both employ heuristic strategies to improve interaction samples. This entails a laborious trial-and-error process to determine the optimal augmentation level for various datasets, making it challenging to achieve consistent improvements.
		
		When compared to the base model on the ML-100K dataset, we do notice a slight drop in recommendation accuracy (i.e., $K$=20 and $K$=30). One possible reason is that there are relatively sparse interactions within this dataset, which may lead model to unintentionally remove some valuable interactions during the adversarial learning process. Nevertheless, when considering the gains in fairness performance, the reduction in accuracy remains minimal.
		\subsection{Ablation Study (RQ2)}

		\begin{figure*}[t]
			\vspace{-0pt}
			\centering
			\setlength{\fboxrule}{0.pt}
			\setlength{\fboxsep}{0.pt}
			\fbox{
				\includegraphics[width=0.75\linewidth]{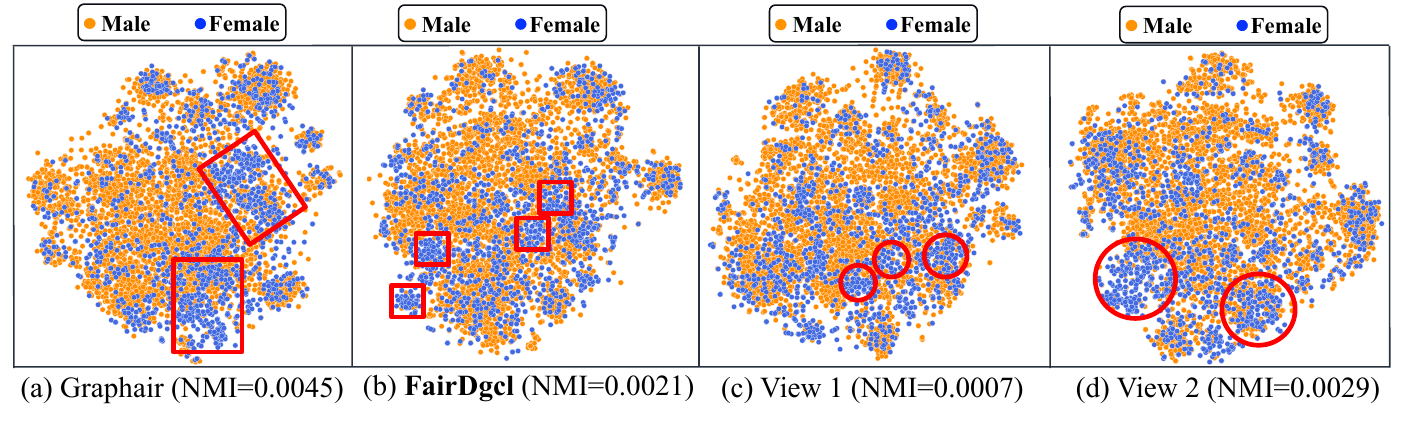}
			}
			\caption{Visualization of learned user embedding on ML-1M dataset for Graphair and FairDgcl. We use Normalized
				Mutual Information (NMI) to evaluate the clustering effect, where a smaller value indicates more dispersed embedding.
			}
			\label{vis}
		\end{figure*}
	
					\begin{figure}[t]
		\vspace{-0pt}
		\centering
		\setlength{\fboxrule}{0.pt}
		\setlength{\fboxsep}{0.pt}
		\fbox{
			\includegraphics[width=1\linewidth]{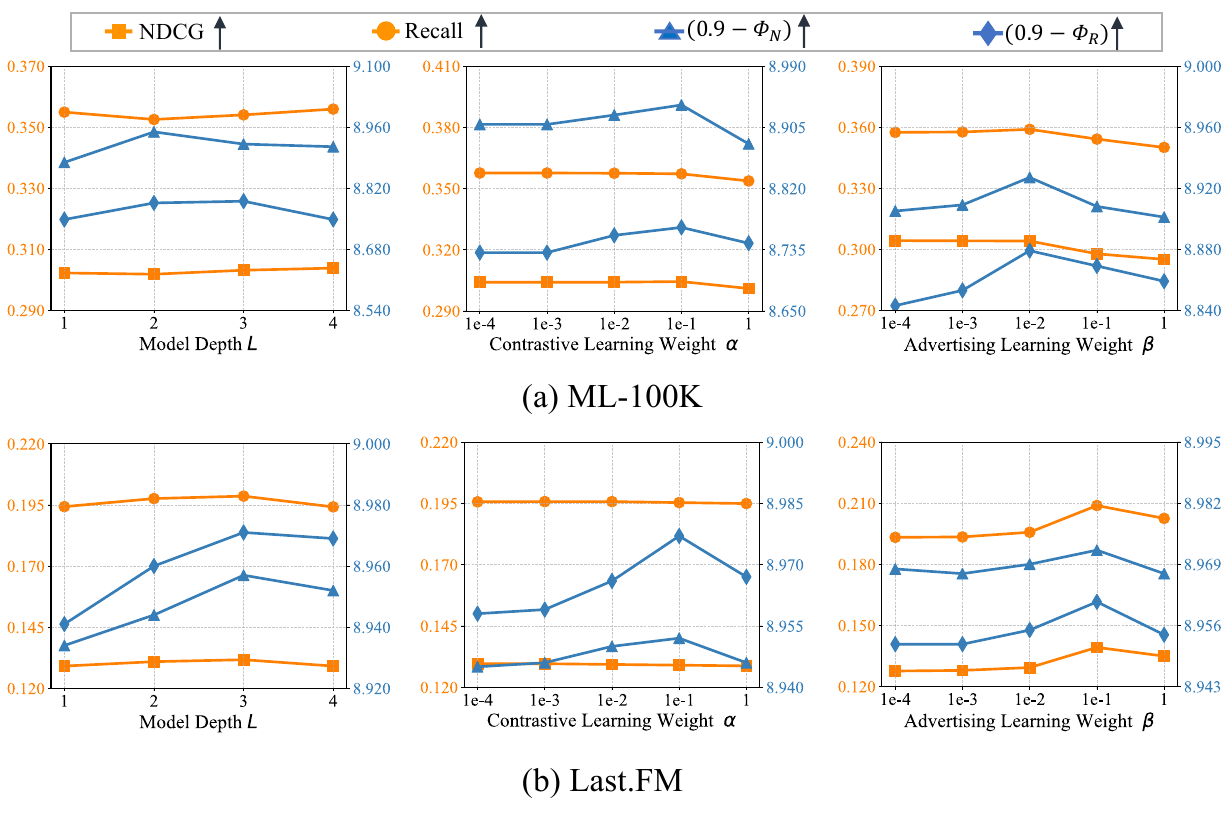}
		}
		\vspace{-15pt}
		\caption{ The recommendation accuracy and fairness performances of our FairDgcl w.r.t varying model depth $L$, weight $\alpha$, and weight $\beta$ on ML-100K and Last.FM datasets, evaluated using the Top-20 criteria.
		}
		\label{para}
		\vspace{-15pt}
	\end{figure}
		To investigate the specific contributions of different components, we conduct a more in-depth analysis of FairDgcl and perform ablation studies. 
		It should be noted that if  the module ``RM” or ``GM”  are removed, FairDgcl will degrade to Graphair. The ablation results are presented in Figure~\ref{ablation}.
		
		First, we can observe that removing specific modules consistently results in a decrease in the model's fairness and accuracy across most scenarios. This emphasizes the crucial role played by various model components in determining the overall performance results. More specifically, within the majority of datasets, the decrease in fairness exhibits a more significant impact compared to the decrease in accuracy. This is unsurprising, given that these modules are primarily crafted with a focus on enhancing fairness.
		
		
		Then, when delving deeper into the comparison between FairDgcl and its variants on ML-1M, a notable observation emerges: the accuracy performance exhibits inconsistency when compared to other datasets. 
		The decrease in accuracy is likely due to a greater variance in gender distribution when compared to other datasets. Eliminating the fairness-related module could potentially strengthen the model's focus on serving the primary user groups. This could lead to better prediction accuracy for these users, ultimately contributing to an overall improvement in accuracy.


		\subsection{Parameter Sensitivity (RQ3)} 
		In this section, we explore FairDgcl's sensitivity to three key hyper-parameters: model depth ($L$), contrastive learning weight ($\alpha$), and advertising learning weight ($\beta$). Due to limited space, we provide experimental results for ML-100K and Last.FM datasets in Figure~\ref{para}. 
		\subsubsection{\textbf{Impact of Model Depth $L$.}}
		Firstly, we investigate the influence of aggregation layer number, represented as $L$. More aggregation layers enable the model to capture deeper collaborative signals, improving its ability to discern user behavior patterns. However, this also raises computational demands and the risk of over-smoothing, we perform experiments with $L$  in $\{1, 2, 3, 4\}$.
		
		From Figure~\ref{para}, we notice that as the layer $L$ increases, fairness performance gradually improves. This improvement can be attributed to the filtering out of more biased interaction signals.
		However, as $L$ continues to rise, the performance eventually stabilizes and may even exhibit a slight decrease. A possible reason is that the model parameters are increasing, which may lead to overfitting.
		It is worth noting that alterations in model depth have a relatively minor effect on model accuracy in comparison to fairness changes.
		We speculate that this arises from the capability of our proposed information enrichment mechanism to directly capture the most valuable information from user-item interactions.


		\subsubsection{\textbf{Impact of Contrastive Learning Weight $\alpha$ }}
		Secondly, we study the impact  of $\alpha$, a parameter that regulates the intensity of contrastive learning. In this analysis, we manipulate the value of $\alpha$, selecting from the set $\{0.0001, 0.001, 0.01, 0.1, 1\}$. The experimental results are shown in second column of Figure~\ref{para}.
		Notably, we observe that optimal fairness performance is attained when $\alpha = 0.1$ for both of these datasets. As $\alpha$ continues to increase beyond this point, we observe a decline in both accuracy and fairness, suggesting the significance of selecting an appropriate value for $\alpha$ in our model.

		
		\subsubsection{\textbf{Impact of Advertising Learning Weight $\beta$ }}
		Finally, we also analyze the impact of the weight $\beta$, which controls the strength of advertising learning.
		Analogously, we vary its values within the set $\{0.0001, 0.001, 0.01, 0.1, 1\}$.
		The results are shown in the last column of Figure~\ref{para}. 
		In the ML-100K dataset, we observe that the optimal trade-off between accuracy and fairness is achieved at $\beta = 0.01$. Beyond this threshold, the fairness and accuracy begin to gradually decrease. Similarly, for Last.FM dataset, the model reaches its ideal trade-off  at $\beta = 0.1$. This suggests that a higher $\beta$ is more favorable for optimizing advertising learning loss in FairDgcl.
		

		\subsection{Visualisation Analysis (RQ4)} 
		To provide a clearer understanding of our model's superior performance, we perform visualization analysis of our proposed FairDgcl and the baseline Graphair. Specifically, we project user embeddings into a 2-D space based on the gender attribute of the ML-1M using the t-SNE tool. 
		The results are shown in Figure~\ref{vis}.
		
		As shown in Figure~\ref{vis}, FairDgcl exhibits a relatively even distribution of users across genders compared to Graphair, evident in the red rectangle. This indicates reduced sensitivity to sensitive attributes in the learned embeddings, achieved through adaptive unfairness mitigation via diverse augmentation strategies, resulting in a uniform clustering effect.
		Furthermore, we compare the model's embeddings under different augmentation views. In general, both augmentation models produce relatively dispersed representations compared to baseline, indicating the effectiveness of augmentation strategies. However, subtle differences still exist between them, with view 2 showing a larger aggregation range compared to view 1 (i.e., denoted by red circles).
		This discrepancy may be attributed to the fact that the generative model (i.e., view 2) focuses on reconstructing unbiased features, preserving more of the original node information. In contrast, the recognition model (i,e., view 1) aims to eliminate unfair connections, effectively removing more biased links and achieving fairer outcomes.

		
		\section{Conclusion AND FUTURE WORK}
		In this paper, we study the user-fairness problem in recommender systems from the data-centric  perspective, and propose a dynamic contrastive learning framework FairDgcl. It leverages an adversarial contrastive network, comprising a view generator and a view discriminator,  to automatically learn and genera fair augmented strategies through a minimax game. Furthermore, we integrate two learnable models to create contrastive views that dynamically fine-tune the augmentation strategies. Finally, 
		the theoretical proofs and experimental results on four real-world datasets clearly show the effectiveness of proposed FairDgcl.
		
		While this paper takes an initial step in applying dynamic data augmentation to enhance recommendation fairness, there remains  room for improvement. For example, our work primarily focuses on binary sensitive attributes, which opens up opportunities for further research into managing a broader range of attribute categories. Additionally, our research assumes the availabiity of sensitive labels, yet real-world scenarios often lack such labels due to user privacy and other constraints. Hence, exploring the application of our framework in scenarios where labels are unavailable is a key area for future research.

\bibliographystyle{IEEEtran}
\bibliography{sample-base}
\begin{IEEEbiography}[{\includegraphics[width=1in,height=1.25in,clip,keepaspectratio]{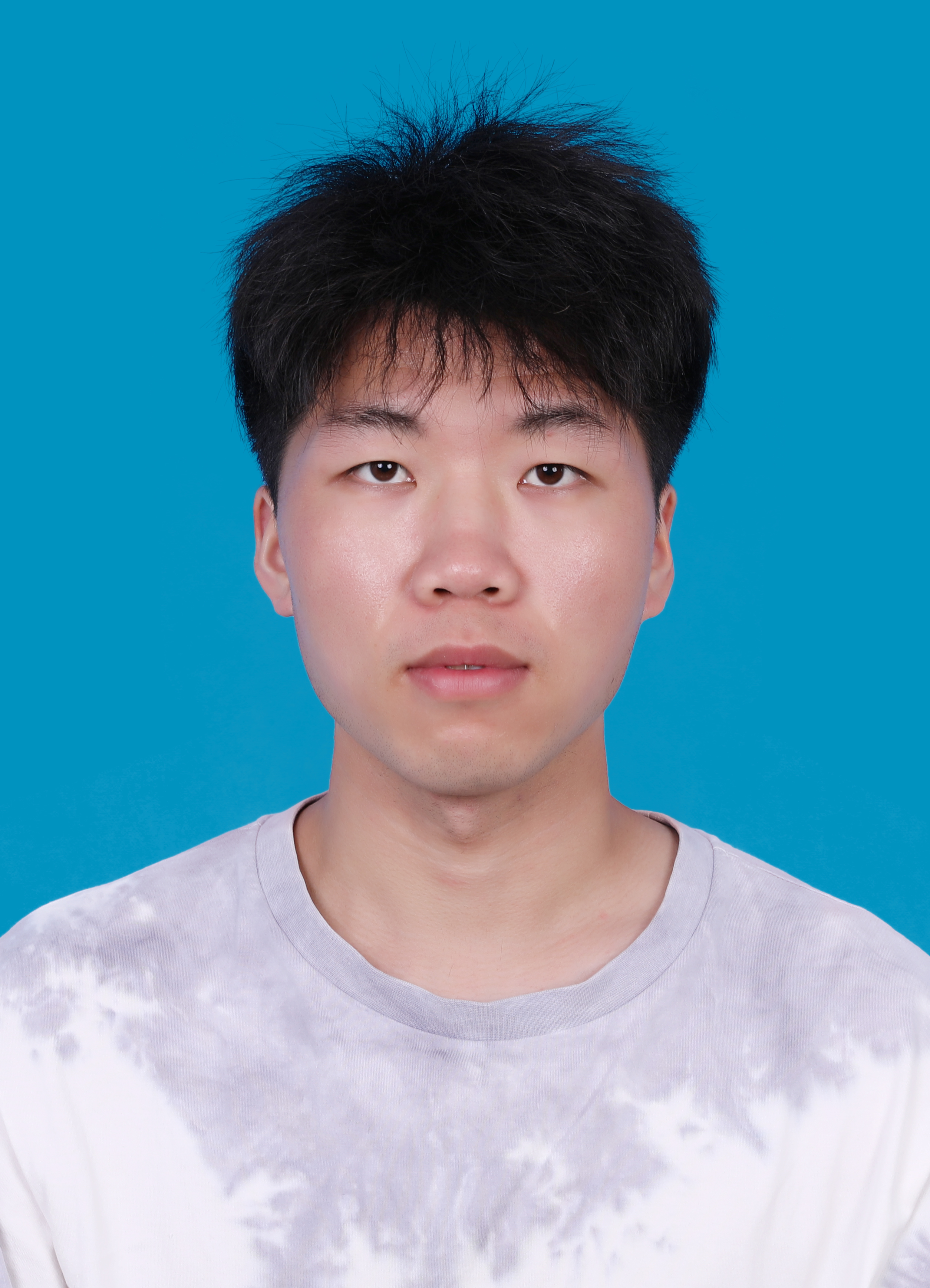}}]{Wei Chen}
is currently pursuing his Ph.D. degree in the Institute of Artificial Intelligence, Beihang University, Beijing, China. His main research interests include recommender system and natural language processing. 
\end{IEEEbiography}
\vspace{-40pt}

\begin{IEEEbiography}[{\includegraphics[width=1in,height=1.25in,clip,keepaspectratio]{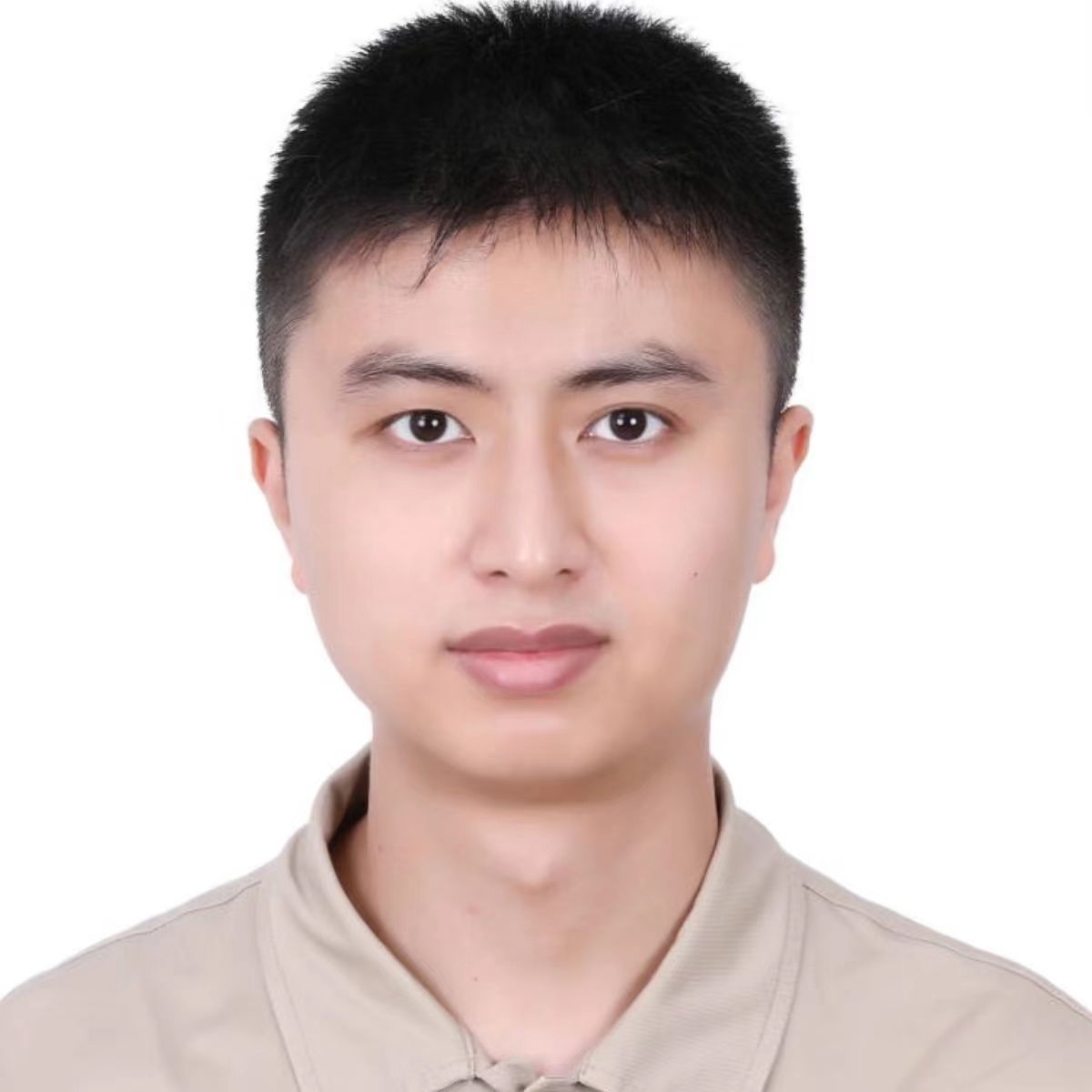}}]{Meng Yuan} is currently pursuing his Ph.D. degree in the Institute of Artificial Intelligence, Beihang University, Beijing, China. His main research interests include recommender system, hyperbolic geometry and machine learning. 
\end{IEEEbiography}
\vspace{-40pt}

\begin{IEEEbiography}[{\includegraphics[width=1in,height=1.25in,clip,keepaspectratio]{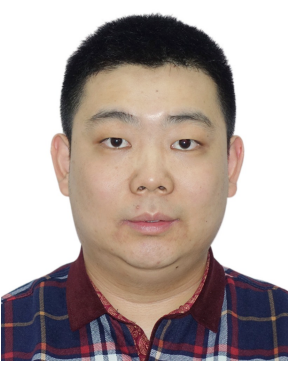}}]{Zhao Zhang}
	 is a research associate at the
	Institute of Computing Technology, Chinese
	Academy of Sciences, Beijing, China. He received the B.E. degree in Computer Science
	and Technology from the Beijing Institute of
	Technology (BIT) in 2015, and Ph.D. degree in
	the Institute of Computing Technology, Chinese
	Academy of Sciences in 2021. His current research interests include data mining and knowledge graphs.
\end{IEEEbiography}
\vspace{-40pt}

\begin{IEEEbiography}[{\includegraphics[width=1in,height=1.25in,clip,keepaspectratio]{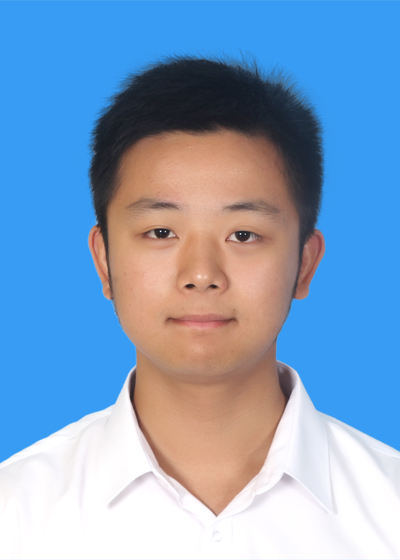}}]{Ruobing Xie}
is a senior researcher of WeChat,
Tencent. He received his BEng degree in 2014
and his master degree in 2017 from the Department of Computer Science and Technology,
Tsinghua University. His research interests include recommender system, knowledge graph,
and natural language processing. He has published over 60 papers in top-tier conferences
and journals including KDD, WWW, SIGIR, ACL,
AAAI, NeurIPS and TKDE. 
\end{IEEEbiography}

\vspace{-40pt}
\begin{IEEEbiography}[{\includegraphics[width=1in,height=1.25in,clip,keepaspectratio]{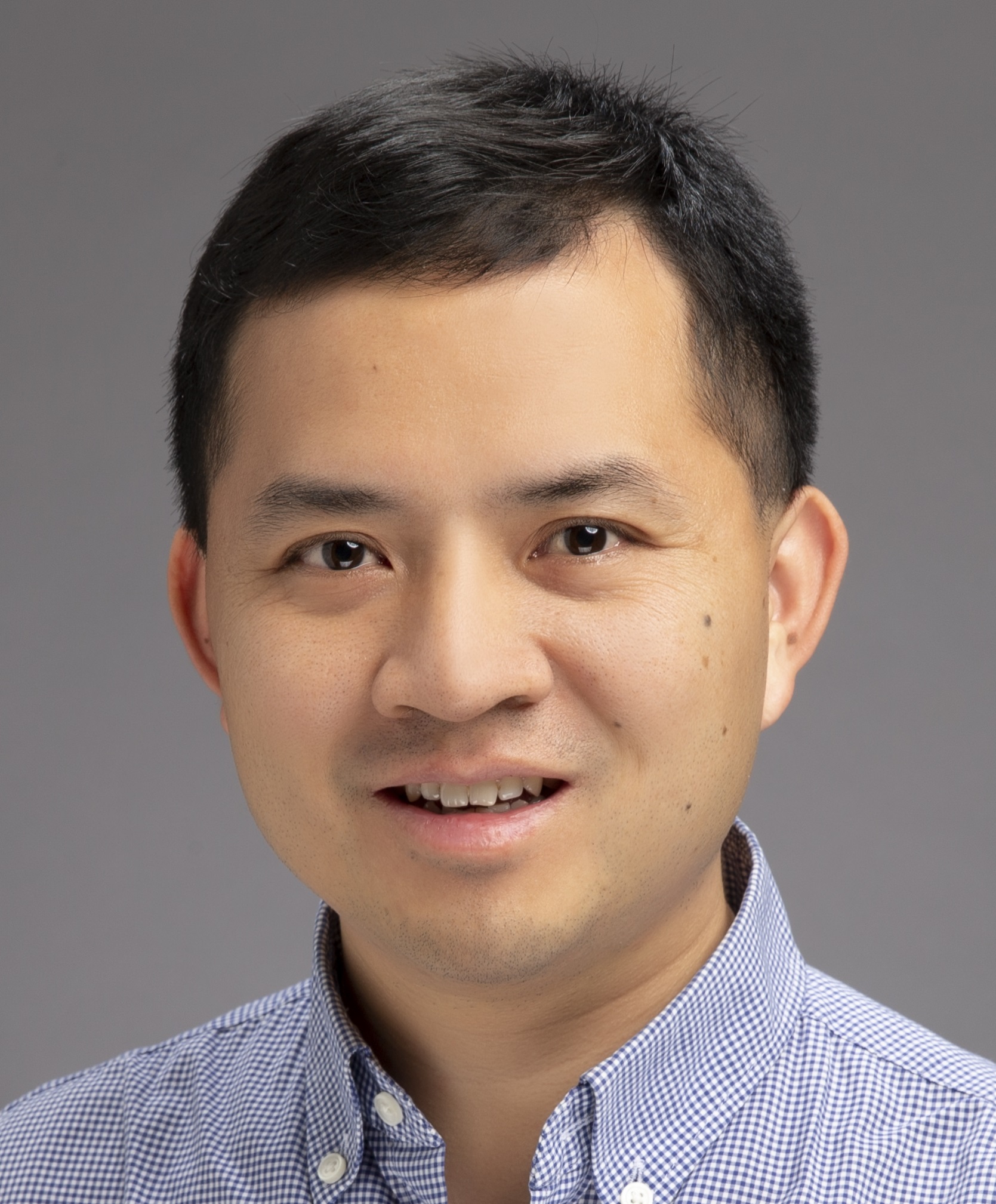}}]{Fuzhen Zhuang}
	is a professor in Institute of Artificial Intelligence, Beihang University. His research interests include transfer learning, machine learning, data mining, multi-task learning and recommendation systems. He has published over 100 papers in the prestigious refereed journals and conference proceedings, such as Nature Communications, TKDE, Proc. of IEEE, TNNLS, TIST, KDD, WWW, SIGIR, NeurIPS, AAAI, and ICDE. 
\end{IEEEbiography}

\vspace{-40pt}
\begin{IEEEbiography}[{\includegraphics[width=1in,height=1.25in,clip,keepaspectratio]{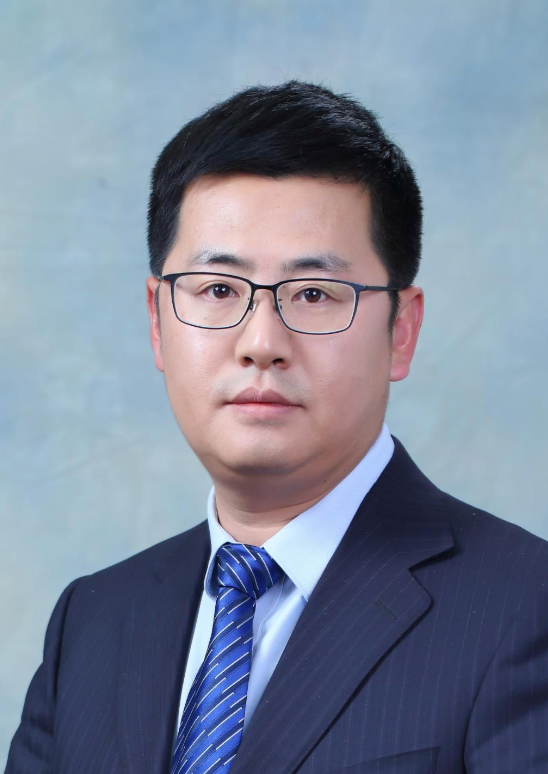}}]{Deqing Wang} received  the  PhD  degree  in  computer  science  from  Beihang  University,  China  in
	2013. He is currently an Associate Professor with
	the  School  of  Computer  Science  and  the  Deputy
	Chief  Engineer  with  the  National  Engineering
	Research  Center  for  Science  Technology
	Resources  Sharing  and  Service,  Beihang
	University,  China.  His  research  focuses  on  text  categorization  and
	data mining for software engineering and machine learning
\end{IEEEbiography}

\vspace{-40pt}
\begin{IEEEbiography}[{\includegraphics[width=1in,height=1.25in,clip,keepaspectratio]{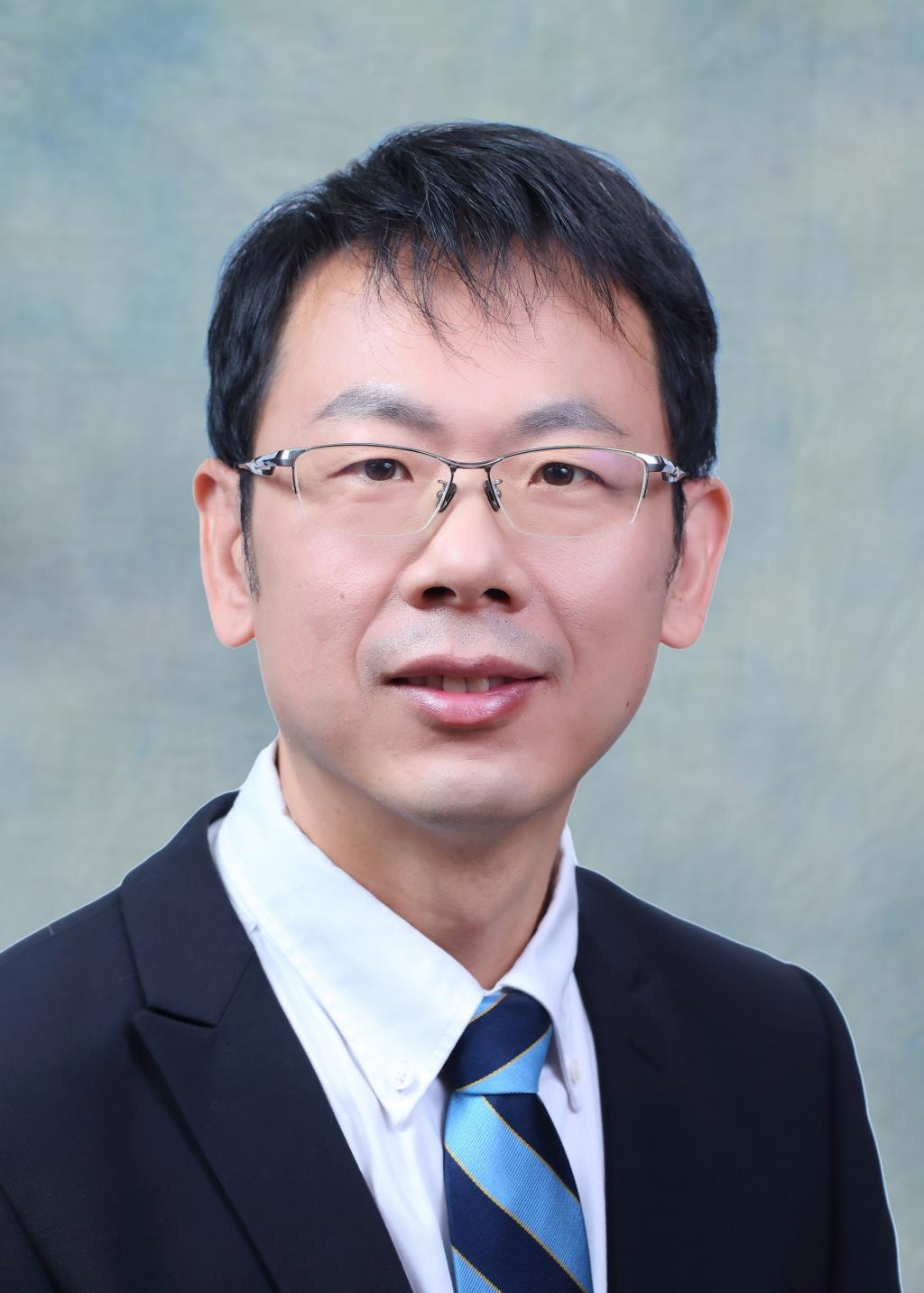}}]{Rui Liu} received his M.S. and Ph.D. degrees in computer
	science from Beihang University, China, in 1995 and
	2011, respectively. He is currently an associate professor
	in the State Key Laboratory of Software Development
	Environment, School of Computer Science and Engineering at Beihang University. 
	Now he is working for National Engineering Research Center for Science and
	Technology Resource Sharing Service of China Portal as
	a chief engineer. His main research interests include
	database, information retrieval and data mining.
\end{IEEEbiography}

\end{document}